\documentclass{article}

\usepackage{bm}
\usepackage{bbding}
\usepackage{mathtools}
\usepackage{amsthm}
\usepackage{amssymb}
\usepackage{enumitem}
\usepackage{relsize}
\usepackage{todonotes}
\usepackage{microtype}
\usepackage{graphicx}
\usepackage{multirow}
\usepackage{subcaption}
\usepackage{booktabs}
\usepackage{algorithm,setspace}
\usepackage{listings}
\usepackage{hyperref}

\usepackage[accepted]{icml2021}

\def\eg{\emph{e.g.}}   
\def\ie{\emph{i.e.}}   
\def\cf{\emph{cf.}}     

\newcommand{\mr}[2]{\multirow{#1}{*}{#2}}
\newcommand{\mc}[3]{\multicolumn{#1}{#2}{#3}}

\renewcommand{\vec}[1]{\bm{#1}}
\newcommand{\mat}[1]{\bm{#1}}

\newcommand{\acc}[2]{#1\textcolor{black!70!white}{\scriptsize{$\pm$#2}}}

\newtheorem{theorem}{Theorem}
\newtheorem{proposition}[theorem]{Proposition}
\newtheorem{lemma}[theorem]{Lemma}

\definecolor{orange}{RGB}{254,128,41}
\definecolor{blue}{RGB}{10,153,201}
\definecolor{green}{RGB}{5,100,18}

\definecolor{codebg}{rgb}{0.96,0.96,0.96}
\definecolor{codebg1}{rgb}{1.0,0.8509803922,0.7490196078}
\definecolor{codebg2}{rgb}{0.7607843137,0.9411764706,0.7843137255}
\definecolor{codered}{rgb}{0.8431372549,0.2274509804,0.2862745098}
\definecolor{codeorange}{rgb}{0.8901960784,0.3843137255,0.03529411765}
\definecolor{codepurple}{rgb}{0.4352941176,0.2588235294,0.7568627451}

\lstdefinestyle{github}{
  language=Python,
  basicstyle=\ttfamily\footnotesize,
  backgroundcolor=\color{codebg},
  keywordstyle=\color{codered}\bfseries,
  emph=[1]{Module,ModuleList,GNN,GCNConv,ScalableGNN},
  emphstyle=[1]{\color{codeorange}\bfseries},
  emph=[2]{append,relu,push_and_pull},
  emphstyle=[2]{\color{codepurple}},
  emph=[3]{__init__,forward},
  emphstyle=[3]{\color{codepurple}\bfseries},
  emph=[4]{w},
  emphstyle=[4]{\color{codebg}},
  captionpos=b,
}

\icmltitlerunning{GNNAutoScale: Scalable and Expressive Graph Neural Networks via Historical Embeddings}

\begin{document}

\twocolumn[
\icmltitle{GNNAutoScale: Scalable and Expressive Graph Neural Networks\\via Historical Embeddings}

\begin{icmlauthorlist}
\icmlauthor{Matthias Fey}{dortmund}
\icmlauthor{Jan Eric Lenssen}{dortmund}
\icmlauthor{Frank Weichert}{dortmund}
\icmlauthor{Jure Leskovec}{stanford}
\end{icmlauthorlist}

\icmlaffiliation{dortmund}{Department of Computer Science, TU Dortmund University}
\icmlaffiliation{stanford}{Department of Computer Science, Stanford University}

\icmlcorrespondingauthor{Matthias Fey}{matthias.fey@udo.edu}

\icmlkeywords{Machine Learning, Deep Learning, Graph Neural Networks, ICML}

\vskip 0.3in
]

\printAffiliationsAndNotice{}

\begin{abstract}
  We present \emph{\underline{G}NN\underline{A}uto\underline{S}cale} (GAS), a framework for scaling arbitrary message-passing GNNs to large graphs.
  GAS prunes entire sub-trees of the computation graph by utilizing historical embeddings from prior training iterations,
  leading to constant GPU memory consumption in respect to input node size without dropping any data.
  While existing solutions weaken the expressive power of message passing due to sub-sampling of edges or non-trainable propagations, our approach is provably able to maintain the expressive power of the original GNN.\@
  We achieve this by providing approximation error bounds of historical embeddings and show how to tighten them in practice.
  Empirically, we show that the practical realization of our framework, \emph{PyGAS}, an easy-to-use extension for \textsc{PyTorch Geometric}, is both fast and memory-efficient, learns expressive node representations, closely resembles the performance of their non-scaling counterparts, and reaches state-of-the-art performance on large-scale graphs.
\end{abstract}

\section{Introduction}%
\label{sec:introduction}

\emph{Graph Neural Networks} (GNNs) capture local graph structure and feature information in a trainable fashion to derive powerful node representations suitable for a given task at hand \citep{Hamilton/2020,Ma/Yang/2020}.
As such, numerous GNNs have been proposed in the past that integrate ideas such as maximal expressiveness \citep{Xu/etal/2019}, anisotropy and attention \citep{Velickovic/etal/2018}, non-linearities \citep{Wang/etal/2019}, or multiple aggregations \citep{Corso/etal/2020} into their message passing formulation.
However, one of the challenges that have so far precluded their wide adoption in industrial and social applications is the difficulty to scale them to large graphs \citep{Frasca/etal/2020}.

While the full-gradient in a GNN is straightforward to compute, assuming one has access to \emph{all} hidden node embeddings in \emph{all} layers, this is not feasible in large-scale graphs due to GPU memory limitations \citep{Ma/Yang/2020}.
Therefore, it is desirable to approximate its full-batch gradient stochastically
by considering only a mini-batch $\mathcal{B} \subseteq \mathcal{V}$ of nodes for loss computation.
However, this stochastic gradient is still expensive to obtain due to the exponentially increasing dependency of nodes over layers; a phenomenon framed as \emph{neighbor explosion} \citep{Hamilton/etal/2017}.
Due to neighbor explosion and since the whole computation graph needs to be stored on the GPU, deeper architectures can not be applied to large graphs.
Therefore, a scalable solution needs to make the memory consumption constant or sub-linear in respect to the number of input nodes. 

Recent works aim to alleviate this problem by proposing various sampling techniques based on the concept of dropping edges \citep{Ma/Yang/2020,Rong/etal/2020}:
\emph{Node-wise sampling} \citep{Hamilton/etal/2017,Chen/etal/2018,Markowitz/etal/2021} recursively samples a fixed number of 1-hop neighbors;
\emph{Layer-wise sampling} techniques independently sample nodes for each layer, leading to a constant sample size in each layer \citep{Chen/etal/2018b,Zou/etal/2019,Huang/etal/2018};
In \emph{subgraph sampling} \citep{Chiang/etal/2019,Zeng/etal/2020a,Zeng/etal/2020b}, a full GNN is run on an entire subgraph $\mathcal{G}[\mathcal{B}]$ induced by a sampled batch of nodes $\mathcal{B} \subseteq \mathcal{V}$.
These techniques get rid of the neighbor explosion problem by sampling the graph but may fail to preserve the edges that present a meaningful topological structure.
Further, existing approaches are either still restricted to shallow networks, non-exchangeable GNN operators or operators with reduced expressiveness.
In particular, they consider only specific GNN operators and it is an open question whether these techniques can be successfully applied to the wide range of GNN architectures available \citep{Velickovic/etal/2018,Xu/etal/2019,Corso/etal/2020,Chen/etal/2020}.
Another line of work is based on the idea of decoupling propagations from predictions, either as a pre- \citep{Wu/etal/2019,Klicpera/etal/2019a,Frasca/etal/2020,Yu/etal/2020} or post-processing step \citep{Huang/etal/2021}.
While this scheme enjoys fast training and inference time, it cannot be applied to any GNN, in particular because the propagation is non-trainable, and therefore reduces model expressiveness.
A different scalability technique is based on the idea of training each GNN layer in isolation \citep{You/etal/2020}.
While this scheme resolves the neighbor explosion problem and accounts for all edges, it cannot infer complex interactions across consecutive layers.

\begin{figure*}[t]
  \centering
  \hfill{}
  \begin{subfigure}[t]{0.28\textwidth}
    \centering
    {\includegraphics[height=4cm]{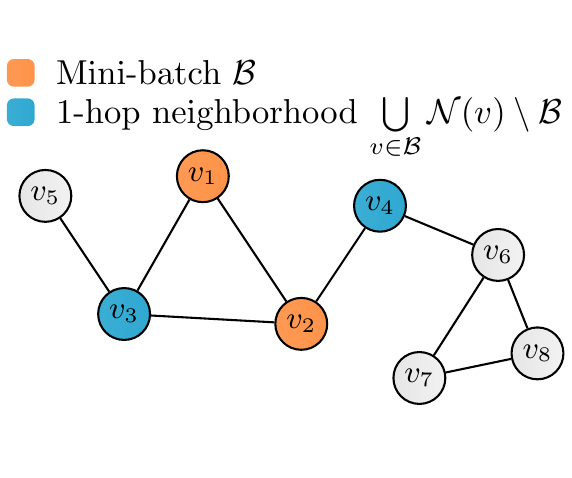}}
    \caption{Mini-batch selection}
  \end{subfigure}
  \hfill{}
  \begin{subfigure}[t]{0.32\textwidth}
    \centering
    {\includegraphics[height=4cm]{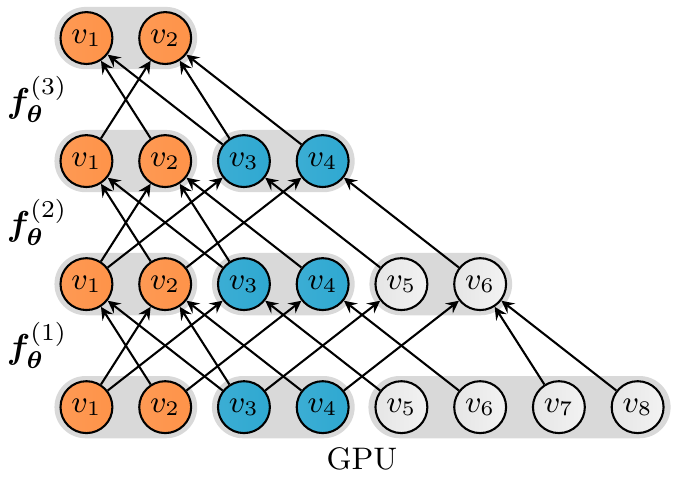}}
    \caption{Original computation graph}\label{fig:original}
  \end{subfigure}
  \hfill{}
  \begin{subfigure}[t]{0.3\textwidth}
    \centering
    {\includegraphics[height=4cm]{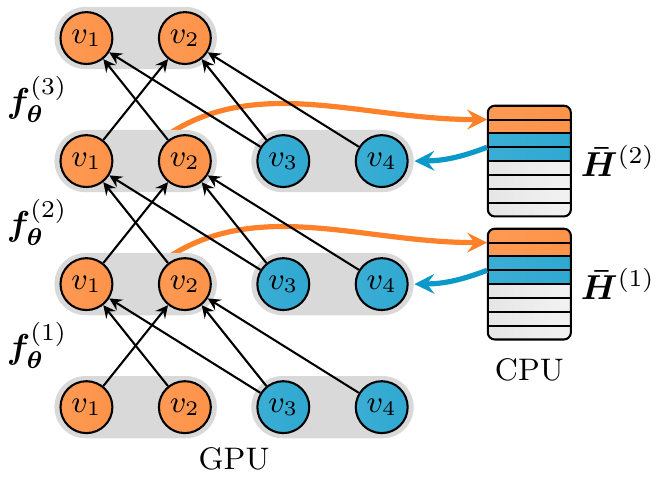}}
    \caption{GAS computation graph}\label{fig:histories}
  \end{subfigure}
  \hfill{}
  \caption{%
    \textbf{Mini-batch processing of GNNs with historical embeddings.}
    \textcolor{orange}{$\blacksquare$} denotes the nodes in the current mini-batch and \textcolor{blue}{$\blacksquare$} represents their direct 1-hop neighbors.
    For a given mini-batch (a), GPU memory and computation costs exponentially increase with GNN depth (b).
    The usage of historical embeddings avoids this problem as it allows to \emph{prune} entire sub-trees of the computation graph, which leads to constant GPU memory consumption in respect to input node size (c).
    Here, nodes in the current mini-batch \emph{push} their updated embeddings to the history $\mat{\bar{H}}^{(\ell)}$, while their direct neighbors \emph{pull} their most recent historical embeddings from $\mat{\bar{H}}^{(\ell)}$ for further processing.
  }\label{fig:cv}
\end{figure*}

Here, we propose the \emph{\underline{G}NN\underline{A}uto\underline{S}cale} (GAS) framework that disentangles the scalability aspect of GNNs from their underlying message passing implementation.
GAS revisits and generalizes the idea of \emph{historical} embeddings \citep{Chen/etal/2018}, which are defined as node embeddings acquired in previous iterations of training, \cf~Figure~\ref{fig:cv}.
For a given mini-batch of nodes, GAS prunes the GNN computation graph so that only nodes inside the current mini-batch and their direct 1-hop neighbors are retained, \emph{independent} of GNN depth.
Historical embeddings act as an offline storage and are used to accurately fill in the inter-dependency information of out-of-mini-batch nodes, \cf~Figure~\ref{fig:histories}.
Through constant memory consumption in respect to input node size, GAS is able to scale the training of GNNs to large graphs, while still accounting for \emph{all} available neighborhood information.

We show that approximation errors induced by historical information are solely caused by the staleness of the history and the Lipschitz continuity of the learned function, and propose solutions for tightening the proven bounds in practice.
Furthermore, we connect scalability with expressiveness and theoretically show under which conditions historical embeddings allow to learn expressive node representations on large graphs.
As a result, GAS is the first scalable solution that is able to keep the existing expressivity properties of the used GNN, which exist for a wide range of models \citep{Xu/etal/2019,Morris/etal/2019,Corso/etal/2020}.

We implement our framework practically as \emph{PyGAS}\footnote{\url{https://github.com/rusty1s/pyg\_autoscale}}, an extension for the \textsc{PyTorch Geometric} library \citep{Fey/Lenssen/2019}, which makes it easy to convert common and custom GNN models into their scalable variants and to apply them to large-scale graphs.
Experiments show that GNNs utilizing GAS achieve the same performances as their (non-scalable) full-batch equivalents (while requiring orders of magnitude less GPU memory), and are able to learn expressive node representations.
Furthermore, GAS allows the application of expressive and hard-to-scale-up models on large graphs, leading to state-of-the-art results on several large-scale graph benchmark datasets.

\section{Scalable GNNs via Historical Embeddings}%
\label{sec:analysis_of_scaling_gnns_up_via_historical_activations}

\paragraph{Background.}

Let $\mathcal{G} = (\mathcal{V}, \mathcal{E})$ or $\mat{A} \in {\{ 0, 1 \}}^{|\mathcal{V}| \times |\mathcal{V}|}$ denote a \emph{graph} with node feature vectors $\vec{x}_v$ for all $v \in \mathcal{V}$.
In this work, we are mostly interested in the task of \emph{node classification}, where each node $v \in \mathcal{V}$ is associated with a label $y_v$, and the goal is to learn a representation $\vec{h}_v$ from which $y_v$ can be easily predicted.
To derive such a representation, GNNs follow a \emph{neural message passing scheme} \citep{Gilmer/etal/2017}.
Formally, the $(\ell + 1)$-th layer of a GNN is defined as (omitting edge features for simplicity) 
\begin{equation}
  \begin{aligned}
    \resizebox{0.105\hsize}{!}{$\vec{h}^{(\ell + 1)}_v$} &
    \resizebox{0.525\hsize}{!}{
      $= \vec{f}^{(\ell+1)}_{\mat{\theta}} \hspace{-2pt} \left( \vec{h}_v^{(\ell)}, {\left\{ \hspace{-5pt} \left\{ \vec{h}_w^{(\ell)} \hspace{-2pt} \right\} \hspace{-5pt} \right\}}_{w \in \mathcal{N}(v)} \right)$}\\
  & \resizebox{0.8\hsize}{!}{
    $= \textsc{Update}^{(\ell+1)}_{\mat{\theta}} \hspace{-2pt} \left( \vec{h}^{(\ell)}_v, \,\bigoplus\limits_{\mathclap{w \in \mathcal{N}(v)}}~
     \textsc{Msg}^{(\ell+1)}_{\mat{\theta}} \big( \vec{h}_w^{(\ell)}, \vec{h}_v^{(\ell)} \big) \right)$}
  \end{aligned}
  \label{eq:gnn}
\end{equation}
where $\vec{h}^{(\ell)}_v$ represents the embedding of node $v$ obtained in layer $\ell$ and $\mathcal{N}(v)$ defines the neighborhood set of $v$.
We initialize $\vec{h}^{(0)}_v = \vec{x}_v$.
Here, $\vec{f}_{\mat{\theta}}^{(\ell + 1)}$ operates on \emph{multisets} $\{ \hspace{-3pt} \{ \ldots \} \hspace{-3pt} \}$ and
can be decomposed into differentiable $\textsc{Message}^{(\ell)}_{\mat{\theta}}$ and $\textsc{Update}^{(\ell)}_{\mat{\theta}}$ functions parametrized by weights $\vec{\theta}$, as well as permutation-invariant aggregation functions $\bigoplus$, \eg~taking the sum, mean or maximum of features \citep{Fey/Lenssen/2019,Gilmer/etal/2017,Qi/etal/2017,Wang/etal/2019,Xu/etal/2019,Kipf/Welling/2017,Velickovic/etal/2018,Hamilton/etal/2017,Klicpera/etal/2019a,Chen/etal/2020,Xu/etal/2018}.
Our following scalability framework is based on the general message passing formulation given in Equation~\eqref{eq:gnn} and thus is applicable to this wide range of different GNN operators.

\paragraph{Historical Embeddings.}%
\label{par:our_approach.}

Let $\vec{h}^{(\ell)}_v$ denote the node embedding in layer $\ell$ of a node $v \in \mathcal{B}$ in a mini-batch $\mathcal{B} \subseteq \mathcal{V}$.
For the general message scheme given in Equation~\eqref{eq:gnn}, the execution of $\vec{f}^{(\ell+1)}_{\mat{\theta}}$ can be formulated as:
\begin{equation}
\begin{aligned}
  &\resizebox{0.64\hsize}{!}{
  $\vec{h}_v^{(\ell+1)} = \vec{f}^{(\ell+1)}_{\mat{\theta}} \Big( \vec{h}_v^{(\ell)}, {\left\{ \hspace{-5pt} \left\{ \vec{h}_w^{(\ell)} \right\} \hspace{-5pt} \right\}}_{w \in \mathcal{N}(v)} \Big)$
  }\\
  &\resizebox{0.9\hsize}{!}{
  $= \vec{f}^{(\ell+1)}_{\mat{\theta}} \Big( \vec{h}_v^{(\ell)}, {\left\{ \hspace{-5pt} \left\{ \vec{h}_w^{(\ell)} \right\} \hspace{-5pt} \right\}}_{w \in \mathcal{N}(v) \cap \mathcal{B}} \cup {\left\{ \hspace{-5pt} \left\{ \vec{h}_w^{(\ell)} \right\} \hspace{-5pt} \right\}}_{w \in \mathcal{N}(v) \setminus \mathcal{B}} \, \Big)$
  }\\
  &\resizebox{0.9\hsize}{!}{
    $\approx \vec{f}^{(\ell+1)}_{\mat{\theta}} \Big( \vec{h}_v^{(\ell)}, {\left\{ \hspace{-5pt} \left\{ \vec{h}_w^{(\ell)} \right\} \hspace{-5pt} \right\}}_{w \in \mathcal{N}(v) \cap \mathcal{B}} \cup \hspace{-4pt} \underbrace{{\left\{ \hspace{-5pt} \left\{ \vec{\bar{h}}_w^{(\ell)} \right\} \hspace{-5pt} \right\}}_{w \in \mathcal{N}(v) \setminus \mathcal{B}}}_{\textrm{\footnotesize Historical embeddings}} \hspace{-4pt} \Big)$
  }\\
\end{aligned}
\label{eq:gnn_approx}
\end{equation}
Here, we separate the neighborhood information of the multiset into \emph{two} parts: \textbf{(1)} the local information of neighbors $\mathcal{N}(v)$ which are part of the current mini-batch $\mathcal{B}$, and \textbf{(2)} the information of neighbors which are not included in the current mini-batch.
For out-of-mini-batch nodes, we approximate their embeddings via historical embeddings acquired in previous iterations of training \citep{Chen/etal/2018}, denoted by $\vec{\bar{h}}_w^{(\ell)}$.
After each step of training, the newly computed embeddings $\vec{h}^{(\ell + 1)}_v$ are pushed to the history and serve as historical embeddings $\vec{\bar{h}}_w^{(\ell + 1)}$ in future iterations.
The separation of in-mini-batch nodes and out-of-mini-batch nodes, and their approximation via historical embeddings represent the foundation of our GAS framework.

A high-level illustration of its computation flow is visualized in Figure~\ref{fig:cv}.
Figure~\ref{fig:original} shows the original data flow without historical embeddings.
The required GPU memory increases as the model gets deeper.
After a few layers, embeddings for the entire input graph need to be stored, even if only a mini-batch of nodes is considered for loss computation.
In contrast, historical embeddings eliminate this problem by approximating entire sub-trees of the computation graph, \cf~Figure~\ref{fig:histories}.
The required historical embeddings are pulled from an off\-line storage, instead of being re-computed in each iteration, which keeps the required information for each batch local.
For a single batch $\mathcal{B} \subseteq \mathcal{V}$, the GPU memory footprint for one training step is given by $\mathcal{O}(|\bigcup_{v \in \mathcal{B}} \mathcal{N}(v) \cup \{ v \}| \cdot L)$ and thus only scales linearly with the number of layers $L$.
The majority of data (the histories) can be stored in RAM or hard drive storage rather than GPU memory.

In the following, we are going to use $\vec{\tilde{h}}_v^{(\ell)}$ to denote embeddings estimated via GAS (line 3 of Equation~\eqref{eq:gnn_approx}) to differentiate them from the exact embeddings obtained without historical approximation (line 1 of Equation~\eqref{eq:gnn_approx}).
In contrast to existing scaling solutions based on sub-sampling edges, the usage of historical embeddings as utilized in GAS provides the following additional advantages:

\textbf{(1) GAS trains over all the data:} In GAS, a GNN will make use of all available graph information, \ie~\emph{no} edges are dropped, which results in lower variance and more accurate estimations (since $\| \vec{\bar{h}}^{(\ell)}_v - \vec{h}_v^{(\ell)} \| \ll \| \vec{h}_v^{(\ell)} \|$).
Importantly, for a single epoch and layer, each edge is still only processed once, putting its time complexity $\mathcal{O}(|\mathcal{E}|)$ on par with its full-batch counterpart.
Notably, more accurate estimations will further strengthen gradient estimation during backpropagation.
Specifically, the model parameters will be updated based on the node embeddings of \emph{all} neighbors since $\partial \vec{\tilde{h}}^{(\ell+1)}_v \hspace{-3pt} / \partial \mat{\theta}$ also depends on $\{ \hspace{-3pt} \{ \vec{\bar{h}}_w^{(\ell)} \colon w \in \mathcal{N}(v) \setminus \mathcal{B} \} \hspace{-3pt} \}$.

\textbf{(2) GAS enables constant inference time complexity:} The time complexity of model inference is reduced to a constant factor, since we can directly use the historical embeddings of the last layer to derive predictions for test nodes.

\textbf{(3) GAS is simple to implement:} Our scheme does not need to maintain recursive layer-wise computation graphs, which makes its overall implementation straightforward and comparable to full-batch training.
Only minor modifications are required  to \emph{pull} information from and \emph{push} information to the histories, \cf~our training algorithm in the appendix.

\textbf{(4) GAS provides theoretical guarantees:}
In particular, if the model weights are kept fixed, $\vec{\tilde{h}}_v^{(\ell)}$ eventually equals $\vec{h}^{(\ell)}_v$ after a fixed amount of iterations \citep{Chen/etal/2018}.

\section{Approximation Error and Expressiveness}%
\label{sec:analysis}

The advantages of utilizing historical embeddings $\vec{\bar{h}}_v^{(\ell)}$ to compute an approximation $\vec{\tilde{h}}_v^{(\ell)}$ of the exact embedding $\vec{h}_v^{(\ell)}$ come at the cost of an approximation error $\| \vec{\tilde{h}}_v^{(\ell)} - \vec{h}_v^{(\ell)} \|$, which can be decomposed into two sources of variance:
\textbf{(1)} The \emph{closeness} of estimated inputs to their exact values, \ie~$\| \vec{\tilde{h}}_v^{(\ell - 1)} - \vec{h}_v^{(\ell - 1)} \| \geq 0$, and \textbf{(2)} the \emph{staleness} of historical embeddings, \ie~$\| \vec{\bar{h}}_v^{(\ell - 1)} - \vec{\tilde{h}}^{(\ell - 1)}_v \| \geq 0$.
In the following, we show concrete bounds for this error, which can be then tightened using specific procedures.
Here, our analysis focuses on arbitrary $\vec{f}^{(\ell)}_{\mat{\theta}}$ GNN layers as described in Equation~\eqref{eq:gnn}, but we restrict both $\textsc{Message}^{(\ell)}_{\mat{\theta}}$ and $\textsc{Update}^{(\ell)}_{\mat{\theta}}$ to model $k$-Lipschitz continuous functions due to their potentially highly non-linear nature. Proofs of all lemmas and theorems can be found in the appendix.

\begin{lemma}\label{lemma1}
  Let $\textsc{Message}^{(\ell)}_{\mat{\theta}}$ and $\textsc{Update}^{(\ell)}_{\mat{\theta}}$ be Lipschitz continuous functions with Lipschitz constants $k_1$ and $k_2$, respectively.
  If, for all $v \in \mathcal{V}$, the inputs are close to the exact input, \ie~$\| \vec{\tilde{h}}^{(\ell - 1)}_{v} - \vec{h}^{(\ell - 1)}_{v} \| \leq \delta$, and the historical embeddings do not run too stale, \ie~$\| \vec{\bar{h}}^{(\ell - 1)}_v - \vec{\tilde{h}}^{(\ell - 1)}_v \| \leq \epsilon$, then the output error is bounded by
  \begin{equation*}
    \| \vec{\tilde{h}}^{(\ell)}_v - \vec{h}^{(\ell)}_v \| \leq \delta\,k_2 + (\delta + \epsilon)\,k_1\,k_2\,|\mathcal{N}(v)|.
  \end{equation*}
\end{lemma}

Due to the behavior of Lipschitz constants in a series of function compositions, we obtain an upper bound that is dependent on $k_1$, $k_2$ and $|\mathcal{N}(v)|$, as well as dependent on the errors $\delta$ and $\epsilon$ of the inputs.
Interestingly, sum aggregation, the most expressive aggregation function \citep{Xu/etal/2019}, introduces a factor of $|\mathcal{N}(v)|$ to the upper bound, while we can obtain a much tighter upper bound for mean or max aggregation, \cf~its proof.
Next, we take a look at the final output error produced by a $L$-layered GNN:\@

\vspace{0.25cm}
\begin{theorem}\label{theorem1}
  Let $\vec{f}^{(L)}_{\mat{\theta}}$ be a $L$-layered GNN, containing only Lipschitz continuous $\textsc{Message}^{(\ell)}_{\mat{\theta}}$ and $\textsc{Update}^{(\ell)}_{\mat{\theta}}$ functions with Lipschitz constants $k_1$ and $k_2$, respectively.
  If, for all $v \in \mathcal{V}$ and all $\ell \in \{1, \ldots, L-1\}$, the historical embeddings do not run too stale, \ie~$\| \vec{\bar{h}}^{(\ell)}_v - \vec{\tilde{h}}^{(\ell)}_v \| \leq \epsilon^{(\ell)}$, then the final output error is bounded by
  \vspace{0.25cm}
  \begin{equation*}
    \| \vec{\tilde{h}}_{v,j}^{(L)} - \vec{h}_{v,j}^{(L)} \| \leq \sum_{\ell = 1}^{L-1} \epsilon^{(\ell)} \, k_1^{L - \ell} \, k_2^{L - \ell} \, {|\mathcal{N}(v)|}^{L - \ell}.
  \end{equation*}
\end{theorem}
Notably, this upper bound does not longer depend on $\| \vec{\tilde{h}}_v^{(\ell)} - \vec{h}_v^{(\ell)} \| \leq \delta^{(\ell)}$, and is instead solely conditioned on the staleness of histories $\| \vec{\bar{h}}_v^{(\ell)} - \vec{\tilde{h}}_v^{(\ell)} \| \leq \epsilon^{(\ell)}$.
However, it depends \emph{exponentially} on the Lipschitz constants $k_1$ and $k_2$ as well as $|\mathcal{N}(v)|$ with respect to the number of layers.
In particular, each additional layer introduces a less restrictive bound since the errors made in the first layers get immediately propagated to later ones, leading to potentially high inaccuracies for histories in deeper GNNs.
We will later propose solutions for tightening the proven bound in practice, allowing the application of GAS to deep and non-linear GNNs.
Furthermore, Theorem~\ref{theorem1} lets us immediately derive an upper error bound of gradients as well, \ie
\begin{equation*}
\| \nabla_{\mat{\theta}} \mathcal{L}(\vec{\tilde{h}}^{(L)}_v) - \nabla_{\mat{\theta}} \mathcal{L}(\vec{h}^{(L)}_v) \| \leq \lambda \| \vec{\tilde{h}}^{(L)}_v - \vec{h}^{(L)}_v \|
\end{equation*}
in case $\mathcal{L}$ is $\lambda$-Lipschitz continuous.
As such, GAS encourages low variance and bias in the learning signal as well.
However, parameters are not guaranteed to converge to the same optimum since we explicitely consider arbitrary GNNs solving non-convex problems \citep{Cong/etal/2020}.

It is well known that the most powerful GNNs adhere to the same representational power as the \emph{Weisfeiler-Lehman (WL) test} \citep{Weisfeiler/Lehman/1968} in distinguishing non-isomorphic structures, \ie~$\vec{h}_v^{(L)} \neq \vec{h}_w^{(L)}$ in case $c_v^{(L)} \neq c_w^{(L)}$ \citep{Xu/etal/2019,Morris/etal/2019}, where $c_v^{(L)}$ denotes a node's coloring after $L$ rounds of color refinement.
However, in order to leverage such expressiveness, a GNN needs to be able to reason about structural differences across neighborhoods directly \emph{during} training.
We now show that GNNs that scale by sampling edges are not capable of doing so:

\begin{proposition}\label{prop3}
  Let $\vec{f}^{(L)}_{\mat{\theta}} \colon \mathcal{V} \to \mathbb{R}^{d}$ be a $L$-layered GNN as expressive as the WL test in distinguishing the $L$-hop neighborhood around each node $v \in \mathcal{V}$.
  Then, there exists a graph $\mat{A} \in {\{0, 1 \}}^{|\mathcal{V}| \times |\mathcal{V}|}$ for which $\vec{f}^{(L)}_{\mat{\theta}}$ operating on a sampled variant $\mat{\tilde{A}}$, $\tilde{a}_{v,w} = \begin{cases} \frac{|\mathcal{N}(v)|}{|\mathcal{\tilde{N}}(v)|}, & \textrm{if } w \in \mathcal{\tilde{N}}(v) \\ 0, & \textrm{otherwise} \end{cases}$, produces a non-equivalent coloring, \ie~$\vec{\tilde{h}}^{(L)}_v \neq \vec{\tilde{h}}^{(L)}_w$ while $c_v^{(L)} = c_w^{(L)}$ for nodes $v, w \in \mathcal{V}$.
\end{proposition}

While sampling strategies lose expressive power due to sub-sampling of edges, scalable GNNs based on historical embeddings are leveraging \emph{all} edges during neighborhood aggregation.
Therefore, a special interest lies in the question if historical-based GNNs are as expressive as their full-batch counterpart.
Here, a maximally powerful \emph{and} scalable GNN needs to fulfill the following two requirements: \textbf{(1)} It needs to be as expressive as the WL test in distinguishing non-isomorphic structures, and \textbf{(2)} it needs to account for the approximation error $\| \vec{\bar{h}}_v^{(\ell - 1)} - \vec{h}_v^{(\ell - 1)} \|$ induced by the usage of historical embeddings.
Since it is known that there exists a wide range of maximally powerful GNNs \citep{Xu/etal/2019,Morris/etal/2019,Corso/etal/2020}, we can restrict our analysis to the latter question.
Following upon \citet{Xu/etal/2019}, we focus on the case where input node features are from a countable set $\mathbb{P}^d \subset \mathbb{R}^d$ of bounded size:

\begin{lemma}\label{lemma2}
  Let $\{ \hspace{-3pt} \{ \vec{h}_v^{(\ell - 1)} \colon v \in \mathcal{V} \} \hspace{-3pt} \}$ be a countable multiset such that $\| \vec{h}_v^{(\ell - 1)} - \vec{h}_w^{(\ell - 1)} \| > 2 (\delta + \epsilon)$ for all $v,w \in \mathcal{V}$, $\vec{h}_v^{(\ell - 1)} \neq \vec{h}_w^{(\ell - 1)}$.
  If the inputs are close to the exact input, \ie~\mbox{$\| \vec{\tilde{h}}_v^{(\ell - 1)} - \vec{h}_v^{(\ell - 1)} \| \leq \delta$}, and the historical embeddings do not run too stale, \ie~\mbox{$\| \vec{\bar{h}}_v^{(\ell - 1)} - \vec{\tilde{h}}_v^{(\ell - 1)} \| \leq \epsilon$}, then there exist $\textsc{Message}^{(\ell)}_{\mat{\theta}}$ and $\textsc{Update}^{(\ell)}_{\mat{\theta}}$ functions, such that
  \begin{equation*}
    \| \vec{f}^{(\ell)}_{\mat{\theta}}(\vec{\tilde{h}}_v^{(\ell - 1)}) - \vec{f}^{(\ell)}_{\mat{\theta}}(\vec{h}_v^{(\ell - 1)}) \| \leq \delta + \epsilon
  \end{equation*}
  and
  \begin{equation*}
    \| \vec{f}^{(\ell)}_{\mat{\theta}}(\vec{h}_v^{(\ell - 1)}) - \vec{f}^{(\ell)}_{\mat{\theta}}(\vec{h}_w^{(\ell - 1)}) \| > 2(\delta + \epsilon + \lambda)
  \end{equation*}
  for all $v,w \in \mathcal{V}$, $\vec{h}_v^{(\ell - 1)} \neq \vec{h}_w^{(\ell - 1)}$ and all $\lambda > 0$.
\end{lemma}

Informally, Lemma~\ref{lemma2} tells us that if \textbf{(1)} exact input embeddings are sufficiently far apart from each other and \textbf{(2)} historical embeddings are sufficiently close to the exact embeddings, there exist historical-based GNN operators which can distinguish equal from non-equal inputs.
Key to the proof is that $(\delta+\epsilon)$-balls around exact inputs do not intersect each other and are therefore well separated.
Notably, we do not require $\vec{f}_{\mat{\theta}}^{(\ell)}$ to model strict injectivity since it is sufficient for $\vec{f}_{\mat{\theta}}^{(\ell)}$ to be $2(\delta + \epsilon)$-injective \citep{Seo/etal/2019}.

Following \citet{Xu/etal/2019}, one can leverage MLPs to model and learn such $\textsc{Message}$ and $\textsc{Update}$ functions due to the universal approximation theorem \citep{Hornik/etal/1989,Hornik/1991}.
However, the theory behind Lemma~\ref{lemma2} holds for any maximally powerful GNN operator.
Finally, we can use this insight to relate the expressiveness of scalable GNNs to the WL test color refinement procedure:

\begin{theorem}\label{theorem2}
  Let $\vec{f}^{(L)}_{\mat{\theta}}$ be a $L$-layered GNN in which all $\textsc{Message}^{(\ell)}_{\mat{\theta}}$ and $\textsc{Update}^{(\ell)}_{\mat{\theta}}$ functions fulfill the conditions of Lemma~\ref{lemma2}.
  Then, there exists a map $\phi \colon \mathbb{R}^d \to \Sigma$ so that $\phi(\vec{\tilde{h}}^{(L)}_v) = c^{(L)}_v$ for all $v \in \mathcal{V}$.
\end{theorem}

Theorem~\ref{theorem2} extends the insights of Lemma~\ref{lemma2} to multi-layered GNNs, and indicates that scalable GNNs using historical embeddings are able to distinguish non-isomorphic structures (that are distinguishable by the WL test) directly during training, which is what makes reasoning about structural properties possible.
It should be noted that recent proposals such as \textsc{DropEdge} \citep{Rong/etal/2020} are still applicable for data augmentation and message reduction.
However, through the given theorem, we disentangle scalability and expressiveness from regularization via edge dropping.

While sampling approaches lose expressiveness compared to their original counterparts (\cf~Proposition~\ref{prop3}), Theorem~\ref{theorem2} tells us that, in theory, there exist message passing functions that are as expressive as the WL test in distinguishing non-isomorphic structures while accounting for the effects of approximation in stored embeddings.
In practice, we have two degrees of freedom to tighten the upper bounds given by Lemma~\ref{lemma1} and Theorem~\ref{theorem1}, leading to a lower approximation error and higher expressiveness in return: \textbf{(1)} Minimizing the \emph{staleness} of historical embeddings, and \textbf{(2)} maximizing the \emph{closeness} of estimated inputs to their exact values by controlling the Lipschitz constants of $\textsc{Update}$ and $\textsc{Message}$ functions.
In what follows, we derive a list of procedures to achieve these goals:

\paragraph{Minimizing Inter-Connectivity Between Batches.}%
\label{par:minimizing_inter-connecivity_between_batches.}

As formulated in Equation~\eqref{eq:gnn_approx} in Section~\ref{sec:analysis_of_scaling_gnns_up_via_historical_activations}, the output embeddings of $\vec{f}_{\mat{\theta}}^{(\ell+1)}$ are exact if $|\bigcup_{v \in \mathcal{B}} \mathcal{N}(v) \cup \{ v \}| = |\mathcal{B}|$, \ie~all neighbors of nodes in $\mathcal{B}$ are as well part of $\mathcal{B}$.
However, in practice, this can only be guaranteed for full-batch GNNs.
Motivated by this observation, we aim to minimize the inter-connectivity between sampled mini-batches, \ie~$\min |\bigcup_{v \in \mathcal{B}} \mathcal{N}(v) \setminus \mathcal{B}|$, which minimizes history access, and increases closeness and reduces staleness in return.

Similar to \textsc{Cluster-GCN} \citep{Chiang/etal/2019}, we make use of graph clustering techniques, \eg, \textsc{Metis} \citep{Karypis/Kumar/1998,Dhillon/etal/2007}, to achieve this goal.
It aims to construct partitions over the nodes in a graph such that intra-links within clusters occur much more frequently than inter-links between different clusters.
Intuitively, this results in a high chance that neighbors of a node are located in the same cluster.
Notably, modern graph clustering methods are both fast and scalable with time complexities given by $\mathcal{O}(|\mathcal{E}|)$, and only need to be applied once, which leads to an unremarkable computational overhead in the pre-processing stage.
In general, we argue that the \textsc{Metis} clustering technique is highly scalable, as it is in the heart of many large-scale distributed graph storage layers such as \citep{Zhu/etal/2019,Zheng/etal/2020} that are known scale to billion-sized graphs.
Furthermore, the additional overhead in the pre-processing stage is quickly compensated by an acceleration of training, since the number of neighbors outside of $\mathcal{B}$ is heavily reduced, and pushing information to the histories now leads to contiguous memory transfers.

\paragraph{Enforcing Local Lipschitz Continuity.}%
\label{par:enforcing_local_lipschitz_continuity}

To guide our neural network in learning a function with controllable error, we can enforce its intermediate output layers $\vec{f}_{\mat{\theta}}^{(\ell)}$ to be invariant to small input perturbations.
In particular, following upon \citet{Usama/Chang/2018}, we found it useful to apply the auxiliary loss
\begin{equation}
  \mathcal{L}_{\textrm{reg}}^{(\ell)} = \| \vec{f}_{\mat{\theta}}^{(\ell)} (\vec{\tilde{h}}_v^{(\ell - 1)}) - \vec{f}_{\mat{\theta}}^{(\ell)} (\vec{\tilde{h}}_v^{(\ell - 1)} + \vec{\epsilon}) \|
\end{equation}
in highly non-linear message passing phases, \eg, in \textsc{GIN} \citep{Xu/etal/2019}.
Such regularization enforces equal outputs for small pertubations $\vec{\epsilon} \sim \mathcal{B}_{\delta}(\vec{0})$ inside closed balls of radius $\delta$.
Notably, we do not restrict $\textsc{Update}^{(\ell)}_{\mat{\theta}}$ and $\textsc{Message}^{(\ell)}_{\mat{\theta}}$ to separately model global $k$-Lipschitz continuous functions, but rather aim for local Lipschitz continuity at each $\vec{h}_v^{(\ell - 1)}$ for $\vec{f}^{(\ell)}_{\mat{\theta}}$ as a whole.
For other message passing GNNs, \eg, in \textsc{GCN} \citep{Kipf/Welling/2017}, $L_2$ regularization is usually sufficient to ensure closeness of historical embeddings.
Further, we found gradient clipping to be an effective method to restrict the parameters from changing too fast, regularizing history changes in return.

\section{Related Work}%
\label{sec:related_work}

Our GAS framework utilizes historical embeddings as an affordable approximation.
The idea of historical embeddings was originally introduced in \textsc{VR-GCN} \citep{Chen/etal/2018}.
\textsc{VR-GCN} aims to reduce the variance in estimation during neighbor sampling \citep{Hamilton/etal/2017}, and avoids the need to sample a large amount of neighbors in return.
\citet{Cong/etal/2020} further simplified this scheme into a \emph{one-shot sampling} scenario, where nodes no longer need to recursively explore neighborhoods in each layer.
However, these approaches consider only a specific GNN operator which prevent their application to the wide range of GNN architectures available.
Furthermore, they only consider shallow architectures and do not account for the increasing approximation error induced by deeper and expressive GNNs, which is well observable in practice, \cf~Section~\ref{sub:analysis_of_proposed_optimizations}.

In order to minimize the inter-connectivity between mini-batches, we utilize graph clustering techniques for mini-batch selection, as first introduced in the subgraph sampling approach \textsc{Cluster-GCN} \citep{Chiang/etal/2019}.
\textsc{Cluster-GCN} leverages clustering in order to infer meaningful subgraphs, while we aim to minimize history accesses.
Furthermore, \textsc{Cluster-GCN} limits message passing to intra-connected nodes, and therefore ignores potentially useful information outside the current mini-batch.
This inherently limits the model to learn from nodes nearby.
In contrast, our GAS framework makes use of \emph{all} available neighborhood data for aggregation, and therefore avoids this downside.

\section{PyGAS:\@ Auto-Scaling GNNs in PyG}%
\label{sec:optimizations}

\begin{figure}[t]
  \centering
  {\includegraphics[width=\linewidth]{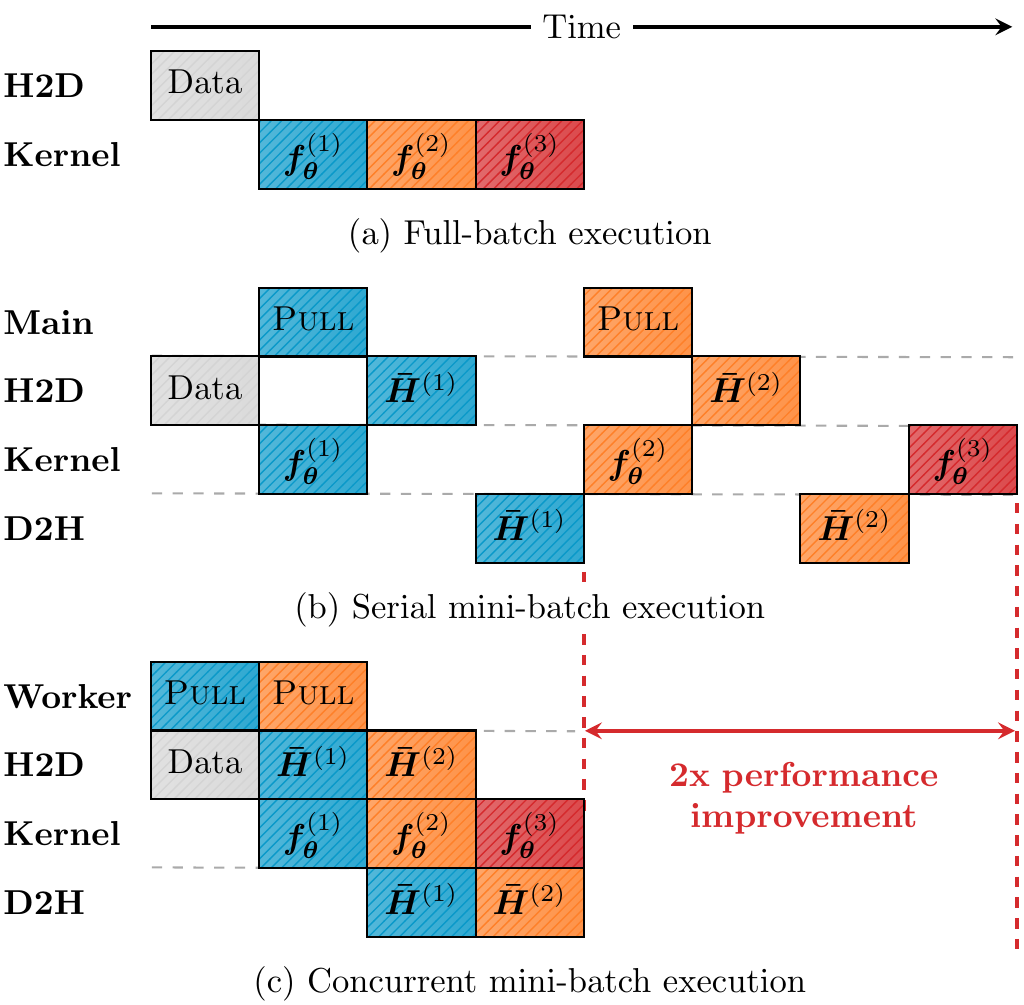}}
  \caption{%
    \textbf{Illustrative runtime performances of a serial and concurrent mini-batch execution in comparison to a full-batch GNN execution.}
    In the full-batch approach (a), all necessary data is first transferred to the device via the \textsc{Host2Device} (H2D) engine, before GNN layers are executed in serial inside the kernel engine.
    As depicted in (b), a serial mini-batch execution suffers from an I/O bottleneck, in particular because each kernel engine has to wait for memory transfers to complete.
    The concurrent mini-batch execution (c) avoids this problem by leveraging an additional worker thread and overlapping data transfers, leading to two times performance improvements in comparison to a serial execution, which is on par with the standard full-batch approach.
  }\label{fig:concurrent}
\end{figure}

\begin{figure*}[t]
  \centering
  \includegraphics[height=0.28cm]{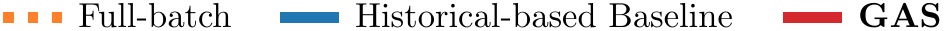}\\
  \hfill{}
  \begin{subfigure}[t]{0.30\textwidth}
    {\includegraphics[height=3.4cm]{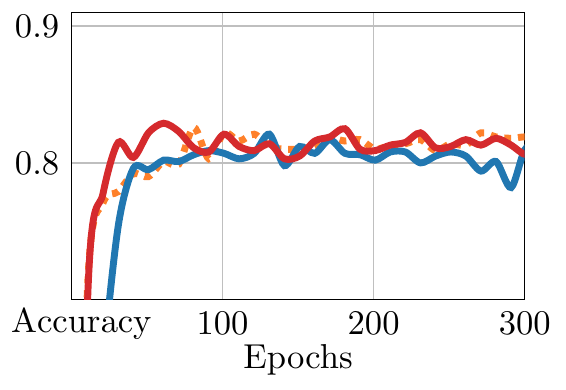}}
    \caption{2-\textsc{GCN} on \textsc{Cora}}\label{fig:gcn}
  \end{subfigure}
  \hfill{}
  \begin{subfigure}[t]{0.30\textwidth}
    {\includegraphics[height=3.4cm]{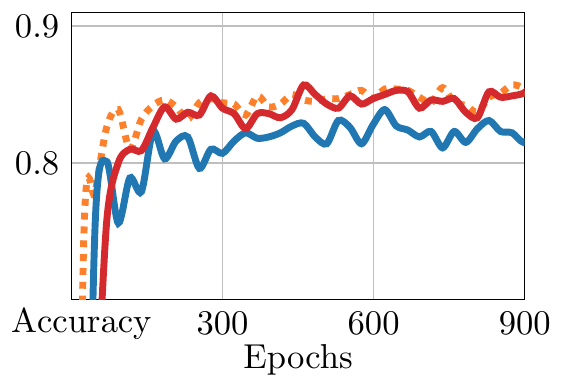}}
    \caption{64-\textsc{GCNII} on \textsc{Cora}}\label{fig:gcnii}
  \end{subfigure}
  \hfill{}
  \begin{subfigure}[t]{0.30\textwidth}
    {\includegraphics[height=3.4cm]{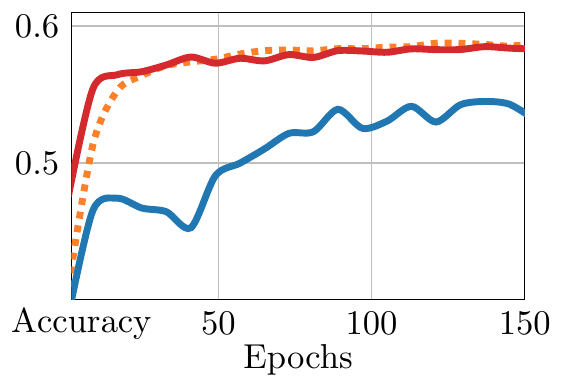}}
    \caption{4-\textsc{GIN} on \textsc{Cluster}}\label{fig:gin}
  \end{subfigure}
  \hfill{}
  \caption{%
    \textbf{Model performance comparison between full-batch, an unoptimized history-based baseline and our GAS approach.}
    In contrast to the historical-based baseline, GAS reaches the quality of full-batch training, especially for (b) deep and (c) expressive models.
    }\label{fig:result2}
\end{figure*}

We condense our GAS framework and theoretical findings into a tool named \emph{PyGAS} that implements all the presented techniques in practice.\footnote{\url{https://github.com/rusty1s/pyg\_autoscale}}
\emph{PyGAS} is built upon \textsc{PyTorch} \citep{Paszke/etal/2019} and utilizes the \textsc{PyTorch Geometric} (PyG) library \citep{Fey/Lenssen/2019}.
It provides an easy-to-use interface to convert common and custom GNN models from \textsc{PyTorch Geometric} into their scalable variants.
Furthermore, it provides a fully deterministic test bed for evaluating models on large-scale graphs.
An example of the interface is shown in the appendix.

\paragraph{Fast Historical Embeddings.}%
\label{par:further_optimizations.}

Our approach accesses histories to account for any data outside the current mini-batch, which requires frequent data transfers to and from the GPU.\@
Therefore, \emph{PyGAS} optimizes pulling from and pushing to histories via \emph{non-blocking} device transfers.
Specifically, we immediately start pulling historical embeddings for each layer asynchronously at the beginning of each optimization step, which ensures that GPUs do not run idle while waiting for memory transfers to complete.
A separate worker thread gathers historical information into one of multiple pinned CPU memory buffers (denoted by \textsc{Pull}), from where it can be transfered to the GPU via the usage of CUDA streams without blocking any CPU or CUDA execution.
Synchronization is done by synchronizing the respective CUDA stream before inputting the transferred data into the GNN layer.
The same strategy is applied for pushing information to the history.
Considering that the device transfer of $\mat{\bar{H}}^{(\ell - 1)}$ is faster than the execution of $\vec{f}_{\mat{\theta}}^{(\ell)}$, this scheme does not lead to any runtime overhead when leveraging historical embeddings and can be twice as fast as its serial non-overlapping counterpart, \cf~Figure~\ref{fig:concurrent}.
We have implemented our non-blocking transfer scheme with custom C++/CUDA code to avoid Python's global interpreter lock.

\section{Experiments}%
\label{sec:experiments}

In this section, we evaluate our GAS framework in practice using \emph{PyGAS}, utilizing 6 different GNN operators and 15 datasets.
Please refer to the appendix for a detailed description of the used GNN operators and datasets, and to our code for hyperparameter configurations.
All models were trained on a single GeForce RTX 2080 Ti (11 GB).
In our experiments, we hold all histories in RAM, using a machine with 64GB of CPU memory.

\subsection{GAS resembles full-batch performance}%
\label{sub:analysis_of_proposed_optimizations}

\begin{table*}
  \centering
  \caption{%
    \textbf{Full-batch vs GAS performance on small transductive graph benchmark datasets across 20 different initializations.}
    Predictive performance of models trained via GAS closely matches those of full-batch gradient descent on all models for all datasets.
  }\label{tab:result1}
  \setlength{\tabcolsep}{5pt}
  \resizebox{0.9\linewidth}{!}{%
  \begin{tabular}{lcccccccccc}
    \mc{9}{c}{\footnotesize{$\dagger$~Results omitted due to unstable performance across different weight initializations, \cf~\citet{Shchur/etal/2018}}} \\
    \toprule
    \mr{2}{\textbf{Dataset}} & \mc{2}{c}{\textsc{GCN}} & \mc{2}{c}{\textsc{GAT}} & \mc{2}{c}{\textsc{APPNP}} & \mc{2}{c}{\textsc{GCNII}} \\
    & \small{Full} & \small{\textbf{GAS}} & \small{Full} & \small{\textbf{GAS}} & \small{Full} & \small{\textbf{GAS}} & \small{Full} & \small{\textbf{GAS}} \\
    \midrule
    \textsc{Cora}             & \acc{81.88}{0.75} & \acc{82.29}{0.76} & \acc{82.80}{0.47} & \acc{83.32}{0.62} & \acc{83.28}{0.60} & \acc{83.19}{0.58} & \acc{85.04}{0.53} & \acc{85.52}{0.39} \\
    \textsc{CiteSeer}         & \acc{70.98}{0.66} & \acc{71.18}{0.97} & \acc{71.72}{0.91} & \acc{71.86}{1.00} & \acc{72.13}{0.73} & \acc{72.63}{0.82} & \acc{73.06}{0.81} & \acc{73.89}{0.48} \\
    \textsc{PubMed}           & \acc{78.73}{1.10} & \acc{79.23}{0.62} & \acc{78.03}{0.40} & \acc{78.42}{0.56} & \acc{80.21}{0.20} & \acc{79.82}{0.52} & \acc{79.72}{0.78} & \acc{80.19}{0.49} \\
    \textsc{Coauthor-CS}      & \acc{91.08}{0.59} & \acc{91.22}{0.45} & \acc{90.31}{0.49} & \acc{90.38}{0.42} & \acc{92.51}{0.47} & \acc{92.44}{0.58} & \acc{92.45}{0.35} & \acc{92.52}{0.31} \\
    \textsc{Coauthor-Physics} & \acc{93.10}{0.84} & \acc{92.98}{0.72} & \acc{92.32}{0.86} & \acc{92.80}{0.61} & \acc{93.40}{0.92} & \acc{93.68}{0.61} & \acc{93.43}{0.52} & \acc{93.61}{0.41} \\
    \textsc{Amazon-Computer}  & \acc{81.17}{1.81} & \acc{80.84}{2.26} & ---$^{\dagger}$   & ---$^{\dagger}$   & \acc{81.79}{2.00} & \acc{81.66}{1.81} & \acc{83.04}{1.81} & \acc{83.05}{1.16} \\
    \textsc{Amazon-Photo}     & \acc{90.25}{1.66} & \acc{90.53}{1.40} & ---$^{\dagger}$   & ---$^{\dagger}$   & \acc{91.27}{1.26} & \acc{91.23}{1.34} & \acc{91.42}{0.81} & \acc{91.60}{0.78} \\
    \textsc{Wiki-CS}          & \acc{79.08}{0.50} & \acc{79.00}{0.41} & \acc{79.44}{0.41} & \acc{79.56}{0.47} & \acc{79.88}{0.40} & \acc{79.75}{0.53} & \acc{79.94}{0.67} & \acc{80.02}{0.43} \\
    \midrule
    $\mat{\Delta}$~\textbf{Mean Accuracy} & \mc{2}{c}{\textbf{+0.13}} & \mc{2}{c}{\textbf{+0.29}} & \mc{2}{c}{\textbf{-0.01}} & \mc{2}{c}{\textbf{+0.29}} \\
    \bottomrule
  \end{tabular}
  }
\end{table*}

First, we analyze how GAS affects the robustness and expressiveness of our method.
We compare GAS against two different baselines: a regular full-batch variant and a history baseline, which naively integrates history-based mini-batch training without any of the additional GAS techniques.
To evaluate, we make use of a shallow 2-layer \textsc{GCN} \citep{Kipf/Welling/2017} and two recent state-of-the-art models: a deep \textsc{GCNII} network with 64 layers \citep{Chen/etal/2020}, and a maximally expressive \textsc{GIN} network with 4 layers \citep{Xu/etal/2019}.
We evaluate those models on tasks for which they are well suitable: classifying academic papers in a citation network (\textsc{Cora}), and identifying community clusters in Stochastic Block Models (\textsc{Cluster}) \citep{Yang/etal/2016,Dwivedi/etal/2020}, \cf~Figure~\ref{fig:result2}.
Since \textsc{Cluster} is a node classification task containing multiple graphs, we first convert it into a super graph (holding all the nodes of all graphs), and partition this super graph using twice as many partitions as there are initial graphs.
It can be seen that especially for deep (64-\textsc{GCNII}, \cf~Figure~\ref{fig:gcnii}) and expressive (4-\textsc{GIN}, \cf~Figure~\ref{fig:gin}) architectures, the naive historical-based baseline fails to reach the desired full-batch performance.
This can be contributed to the high approximation error induced by deep and expressive models.
In contrast, GAS shows far superior performance, reaching the quality of full-batch training in both cases.

In general, we expect the model performances of our GAS mini-batch training to closely resemble the performances of their full-batch counterparts, except for the variance introduced by stochastic optimization (which is, in fact, known to improve generalization \citep{Bottou/Bousquet/2007}).
To validate, we compare our approach against full-batch performances on small transductive benchmark datasets for which full-batch training is easily feasible.
We evaluate on four GNN models that significantly advanced the field of graph representation learning: \textsc{GCN} \citep{Kipf/Welling/2017}, GAT \citep{Velickovic/etal/2018}, \textsc{APPNP} \citep{Klicpera/etal/2019a} and \textsc{GCNII} \citep{Chen/etal/2020}.
For all experiments, we tried to follow the hyperparameter setup of the respective papers as closely as possible and perform an in-depth grid search on datasets for which best performing configurations are not known.
We then apply GAS mini-batch training on the \emph{same} set of hyperparameters.
As shown in Table~\ref{tab:result1}, all models that utilize GAS training perform as well as their full-batch equivalents (with slight gains overall), confirming the practical effectiveness of our approach.
Notably, even for deep GNNs such as \textsc{APPNP} and \textsc{GCNII}, our approach is able to closely resemble the desired performance.

\begin{table*}
  \centering
  \caption{%
    \textbf{Relative performance improvements of individual GAS techniques within a GCNII model.}
    The performance improvement is measured in percentage points in relation to the corresponding model performance obtained by full-batch training.
  }\label{tab:result1_ablation}
  \setlength{\tabcolsep}{5pt}
  \resizebox{0.78\linewidth}{!}{%
  \begin{tabular}{lcccccccc}
    \toprule
    & \mr{2}{\textsc{Cora}} & \mr{2}{\textsc{CiteSeer}} & \mr{2}{\textsc{PubMed}} & \mc{2}{c}{\textsc{Coauthor-}} & \mc{2}{c}{\textsc{Amazon-}} & \mr{2}{\textsc{Wiki-CS}} \\
    & & & & \textsc{CS} & \textsc{Physics} & \textsc{Computer} & \textsc{Photo} & \\
    \midrule
    Baseline       & -3.26 & -5.66 & -3.20 & -0.79 & -0.50 & -5.76 & -4.16 & -3.19 \\
    Regularization & -2.12 & -1.03 & -1.24 & -0.46 & -0.24 & -3.02 & -1.19 & -0.74 \\
    \textsc{Metis} & -1.57 & -3.12 & -1.50 & -0.47 & +0.13 & -2.75 & -1.02 & -0.24 \\
    \textbf{GAS}   & \textbf{+0.48} & \textbf{+0.83} & \textbf{+0.47} & \textbf{+0.07} & \textbf{+0.18} & \textbf{+0.01} & \textbf{+0.18} & \textbf{+0.08} \\
    \bottomrule
  \end{tabular}
  }
\end{table*}

We further conduct an ablation study to highlight the individual performance improvements of our GAS techniques within a \textsc{GCNII} model, \ie~minimizing inter-connectivity and applying regularization techniques.
Table~\ref{tab:result1_ablation} shows the relative performance improvements of individual GAS techniques in percentage points, compared to the corresponding model performance obtained by full-batch training.
Notably, it can be seen that both techniques contribute to resembling full-batch performance, reaching their full strength when used in combination.
We include an additional ablation study for training an expressive \textsc{GIN} model in the appendix.

\subsection{GAS is fast and memory-efficient}%
\label{sub:runtime_analysis}

\begin{table}
  \centering
  \vspace{-0.5cm}
  \caption{%
    \textbf{GPU memory consumption (in GB) and the amount of data used ($\%$) across different GNN execution techniques.}
    GAS consumes low memory while making use of all available neighborhood information during a single optimization step.
  }\label{tab:memory}
  \setlength{\tabcolsep}{0pt}
  \resizebox{\linewidth}{!}{%
    \begin{tabular}{llrrrrrrr}
      \toprule
      & \footnotesize{\textbf{\#\,nodes}} & \mc{2}{c}{\footnotesize{717K}} & \mc{2}{c}{\footnotesize{169K}} & \mc{2}{c}{\footnotesize{2.4M}} \\[-0.1cm]
      & \footnotesize{\textbf{\#\,edges}} & \mc{2}{c}{\footnotesize{7.9M}} & \mc{2}{c}{\footnotesize{1.2M}} & \mc{2}{c}{\footnotesize{61.9M}} \\[-0.05cm]
      & \mr{2}{\textbf{Method}} & \mc{2}{c}{\mr{2}{\textsc{Yelp}}} & \mc{2}{c}{\texttt{ogbn-}} & \mc{2}{c}{\texttt{ogbn-}} \\
      & & & & \mc{2}{c}{\texttt{arxiv}} & \mc{2}{c}{\texttt{products}} \\
      \midrule
      \mr{4}{\rotatebox{90}{\small{2-layer}}}
      & ~Full-batch             & 6.64GB/ & 100\%~~ & 1.44GB/ & 100\%~~ & 21.96GB/ & 100\% \\
      & ~\textsc{GraphSAGE}     & 0.76GB/ &   9\%~~ & 0.40GB/ &  27\%~~ &  0.92GB/ &   2\% \\
      & ~\textsc{Cluster-GCN}   & 0.17GB/ &  13\%~~ & 0.15GB/ &  40\%~~ &  0.16GB/ &  16\% \\
      & ~\textbf{GAS}           & 0.51GB/ & 100\%~~ & 0.22GB/ & 100\%~~ &  0.36GB/ & 100\% \\
      \midrule
      \mr{4}{\rotatebox{90}{\small{3-layer}}}
      & ~Full-batch             & 9.44GB/ & 100\%~~ & 2.11GB/ & 100\%~~ & 31.53GB/ & 100\% \\
      & ~\textsc{GraphSAGE}     & 2.19GB/ &  14\%~~ & 0.93GB/ &  33\%~~ &  4.34GB/ &   5\% \\
      & ~\textsc{Cluster-GCN}   & 0.23GB/ &  13\%~~ & 0.22GB/ &  40\%~~ &  0.23GB/ &  16\% \\
      & ~\textbf{GAS}           & 0.79GB/ & 100\%~~ & 0.34GB/ & 100\%~~ &  0.59GB/ & 100\% \\
      \midrule
      \mr{4}{\rotatebox{90}{\small{4-layer}}}
      & ~Full-batch             & 12.24GB/ & 100\%~~ & 2.77GB/ & 100\%~~ & 41.10GB/ & 100\% \\
      & ~\textsc{GraphSAGE}     &  4.31GB/ &  19\%~~ & 1.55GB/ &  37\%~~ & 11.23GB/ &   8\% \\
      & ~\textsc{Cluster-GCN}   &  0.30GB/ &  13\%~~ & 0.29GB/ &  40\%~~ &  0.29GB/ &  16\% \\
      & ~\textbf{GAS}           &  1.07GB/ & 100\%~~ & 0.46GB/ & 100\%~~ &  0.82GB/ & 100\% \\
      \bottomrule
    \end{tabular}
  }
\end{table}

For training large-scale GNNs, GPU memory consumption will directly dictate the scalability of the given approach.
Here, we show how GAS maintains a low GPU memory footprint while, in contrast to other scalability approaches, accounts for \emph{all} available information inside a GNN's receptive field in a single optimization step.
We compare the memory usage of \textsc{GCN}+\textsc{GAS} training with the memory usage of full-batch \textsc{GCN}, and mini-batch \textsc{GraphSAGE} \citep{Hamilton/etal/2017} and \textsc{Cluster-GCN} \citep{Chiang/etal/2019} training, \cf~Table~\ref{tab:memory}.
Notably, GAS is easily able to fit the required data on the GPU, while memory consumption only increases linearly with the number of layers.
Although \textsc{Cluster-GCN} maintains an overall lower memory footprint than GAS, it will only utilize a fraction of available information inside its receptive field, \ie~$\approx$23\% on average.

\begin{figure}[t]
  \centering
  \vspace{-0.2cm}
  {\includegraphics[height=4.3cm]{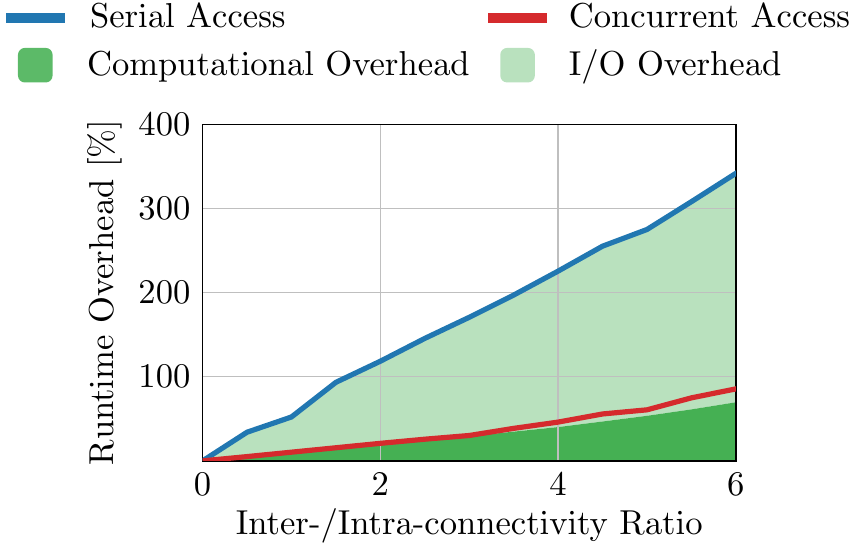}}
  \caption{%
    \textbf{Runtime overhead in relation to the inter-/intra-connectivity ratio of mini-batches, both for serial and concurrent history access patterns.}
    The overall runtime overhead is further separated into computational overhead (overhead of aggregating additional messages) and I/O overhead (overhead of pulling from and pushing to histories).
    Our concurrent memory transfer reduces I/O overhead caused by histories by a wide margin.
  }\label{fig:concurrent2}
\end{figure}

We now analyze how GAS enables large-scale training due to fast mini-batch execution.
Specifically, we are interested in how our concurrent memory transfer scheme (\cf~Section~\ref{sec:optimizations}) reduces the overhead induced by accessing historical embeddings from the offline storage.
For this, we evaluate runtimes of a $4$-layer \textsc{GIN} model on synthetic graph data, which allows fine-grained control over the ratio between inter- and intra-connected nodes, \cf~Figure~\ref{fig:concurrent2}.
Here, a given mini-batch consists of exactly 4,000 nodes which are randomly intra-connected to $60$ other nodes.
We vary the number of inter-connections (connections to nodes outside of the batch) by adding out-of-batch nodes that are randomly inter-connected to $60$ nodes inside the batch.
Notably, the naive serial memory transfer increases runtimes up to 350\%, which indicates that frequent history accesses can cause major I/O bottlenecks.
In contrast, our concurrent access pattern incurs \emph{almost no} I/O overhead \emph{at all}, and the overhead in execution time is solely explained by the computational overhead of aggregating far more messages during message propagation.
Note that in most real-world scenarios, the additional aggregation of history data may only increase runtimes up to 25\%, since most real-world datasets contain inter-/intra-connectivity ratios between $0.1$ and $2.5$, \cf~appendix.
Further, the additional overhead of computing \textsc{Metis} partitions in the pre-processing stage is negligible and is quickly mitigated by faster training times: Computing the partitioning of a graph with 2M nodes takes only about 20--50 seconds (depending on the number of clusters).

Next, we compare runtimes and memory consumption of GAS to the recent \textsc{GTTF} proposal \citep{Markowitz/etal/2021}, which utilizes a fast neighbor sampling strategy based on tensor functionals.
For this, we make use of a 4-layered \textsc{GCN} model with equal mini-batch and receptive field sizes.
As shown in Table~\ref{tab:gttf}, GAS is both faster and consumes less memory than \textsc{GTTF}.
Although \textsc{GTTF} makes use of a fast vectorized sampling procedure, its underlying recursive neighborhood construction still scales \emph{exponentially} with GNN depth, which explains the observable differences in runtime and memory consumption.

\begin{table}
  \centering
  \caption{%
  \textbf{\small Efficiency of \textsc{GCN} with GTTF and GAS.}
  }\label{tab:gttf}
  \setlength{\tabcolsep}{10pt}
  \resizebox{0.9\linewidth}{!}{%
  \begin{tabular}{lrrrr}
    \toprule
    \mr{2}{\textbf{Dataset}} & \mc{2}{c}{\textbf{Runtime} (s)} & \mc{2}{c}{\textbf{Memory} (MB)} \\
    & GTTF & \textbf{GAS} & GTTF & \textbf{GAS} \\
    \midrule
    \textsc{Cora} & 0.077 & \textbf{0.006} & 18.01 & \textbf{2.13} \\
    \textsc{PubMed} & 0.071 & \textbf{0.006} & 28.79 & \textbf{2.19} \\
    \textsc{PPI} & 0.976 & \textbf{0.007} & 134.86 & \textbf{12.37} \\
    \textsc{Flickr} & 1.178 & \textbf{0.007} & 325.97 & \textbf{16.32} \\
    \bottomrule
  \end{tabular}
  }
\end{table}

\subsection{GAS scales to large graphs}%
\label{sub:pygautoscale_scales_to_large_graphs}

In order to demonstrate the scalability and generality of our approach, we scale various GNN operators to common large-scale graph benchmark datasets.
Here, we focus our analysis on GNNs that are notorious hard to scale-up but have the potential to leverage the increased amount of available data to make more accurate predictions.
In particular, we benchmark deep GNNs, \ie~\textsc{GCNII} \citep{Chen/etal/2020}, and expressive GNNs, \ie~{PNA} \citep{Corso/etal/2020}.
Note that it is not possible to run those models in full-batch mode on most of these datasets as they will run out of memory on common GPUs.
We compare with 10 scalable GNN baselines: \textsc{GraphSAGE} \citep{Hamilton/etal/2017}, \textsc{FastGCN} \citep{Chen/etal/2018b}, \textsc{LADIES} \citep{Zou/etal/2019}, \textsc{VR-GCN} \citep{Chen/etal/2018}, \textsc{MVS-GNN} \citep{Cong/etal/2020}, \textsc{Cluster-GCN} \citep{Chiang/etal/2019}, \textsc{GraphSAINT} \citep{Zeng/etal/2020a}, \textsc{SGC} \citep{Wu/etal/2019}, \textsc{SIGN} \citep{Frasca/etal/2020} and \textsc{GBP} \citep{Chen/etal/2020b}.
Since results are hard to compare across different approaches due to differences in frameworks, model implementations, weight initializations and optimizers, we additionally report a shallow \textsc{GCN}+\textsc{GAS} baseline.
GAS is able to train all models on all datasets on a single GPU, while holding corresponding histories in CPU memory.
On the largest dataset, \ie~\texttt{ogbn-products}, this will consume $\approx L \cdot$ 2GB of storage for $L$ layers, which easily fits in RAM on most modern workstations.

\begin{table}
  \centering
  \caption{%
    \textbf{Performance on large graph datasets.}
    GAS is both scalable and general while achieving state-of-the-art performance.
  }\label{tab:result2}
  \setlength{\tabcolsep}{2pt}
  \resizebox{\linewidth}{!}{%
  \begin{tabular}{llccccccc}
    \toprule
    \mc{2}{l}{\footnotesize{\textbf{\#\,nodes}}} & \footnotesize{230K} & \footnotesize{57K} & \footnotesize{89K} & \footnotesize{717K} & \footnotesize{169K} & \footnotesize{2.4M} \\[-0.1cm]
    \mc{2}{l}{\footnotesize{\textbf{\#\,edges}}} & \footnotesize{11.6M} & \footnotesize{794K} & \footnotesize{450K} & \footnotesize{7.9M} & \footnotesize{1.2M} & \footnotesize{61.9M} \\[-0.05cm]
    \mc{2}{l}{\mr{2}{\textbf{Method}}} & \mr{2}{\textsc{Reddit}} & \mr{2}{\textsc{PPI}} & \mr{2}{\textsc{Flickr}} & \mr{2}{\textsc{Yelp}} & \texttt{ogbn-} & \texttt{ogbn-} \\
    & & & & & & \texttt{arxiv} & \texttt{products} \\
    \midrule
    \mc{2}{l}{\textsc{GraphSAGE}}   & 95.40          & 61.20          & 50.10          & 63.40          & 71.49          & 78.70 \\
    \mc{2}{l}{\textsc{FastGCN}}     & 93.70          & ---            & 50.40          & ---            & ---            & ---   \\
    \mc{2}{l}{\textsc{LADIES}}      & 92.80          & ---            & ---            & ---            & ---            & ---   \\
    \mc{2}{l}{\textsc{VR-GCN}}      & 94.50          & 85.60          & ---            & 61.50          & ---            & ---   \\
    \mc{2}{l}{\textsc{MVS-GNN}}     & 94.90          & 89.20          & ---            & 62.00          & ---            & ---   \\
    \mc{2}{l}{\textsc{Cluster-GCN}} & 96.60          & 99.36          & 48.10          & 60.90          & ---            & 78.97 \\
    \mc{2}{l}{\textsc{GraphSAINT}}  & 97.00          & \textbf{99.50} & 51.10          & 65.30          & ---            & 79.08 \\
    \mc{2}{l}{\textsc{SGC}}         & 96.40          & 96.30          & 48.20          & 64.00          & ---            & ---   \\
    \mc{2}{l}{\textsc{SIGN}}        & 96.80          & 97.00          & 51.40          & 63.10          & ---            & 77.60 \\
    \mc{2}{l}{\textsc{GBP}}         & ---            & 99.30          & ---            & \textbf{65.40} & ---            & ---   \\
    \midrule
    \mr{4}{\rotatebox{90}{\small{\,Full-batch}}}
    & \\ [-0.31cm]
    & ~~\textsc{GCN}   & 95.43 & 97.58 & 53.73 & OOM & 71.64 & OOM \\
    & ~~\textsc{GCNII} & OOM   & OOM   & 55.28 & OOM & 72.83 & OOM \\
    & ~~\textsc{PNA}   & OOM   & OOM   & 56.23 & OOM & 72.17 & OOM \\ [0.125cm]
    \midrule
    \mr{3}{\rotatebox{90}{\small{\textbf{GAS}}}}
    & ~~\textsc{GCN}                & 95.45          & 98.92          & 54.00          & 62.94          & 71.68          & 76.66 \\
    & ~~\textsc{GCNII}              & 96.77          & \textbf{99.50} & 56.20          & 65.14          & \textbf{73.00} & 77.24 \\
    & ~~\textsc{PNA}                & \textbf{97.17} & 99.44          & \textbf{56.67} & 64.40          & 72.50          & \textbf{79.91} \\
    \bottomrule
  \end{tabular}
  }
\end{table}

As can be seen in Table~\ref{tab:result2}, the usage of deep and expressive models within our framework advances the state-of-the-art on \textsc{Reddit} and \textsc{Flickr}, while it performs equally well for others, \eg, \textsc{PPI}.
Notably, our approach outperforms the two historical-based variants \textsc{VR-GCN} and \textsc{MVS-GNN} by a wide margin.
Interestingly, our deep and expressive variants reach superior performance than our \textsc{GCN} baseline on \emph{all} datasets, which highlights the benefits of evaluating larger models on larger scale.

\section{Conclusion and Future Work}%
\label{sec:conclusion}

We proposed a general framework for scaling arbitrary message passing GNNs to large graphs without the necessity to sub-sample edges.
As we have shown, our approach is able to train both deep and expressive GNNs in a scalable fashion.
Notably, our approach is \emph{orthogonal} to many methodological advancements, such as unifying GNNs and label propagation \citep{Shi/etal/2020}, graph diffusion \citep{Klicpera/etal/2019b}, or random wiring \citep{Valsesia/etal/2020}, which we like to investigate further in future works.
While our experiments focus on node-level tasks, our work is technically able to scale the training of GNNs for edge-level and graph-level tasks as well.
However, this still needs to be verified empirically.
Another interesting future direction is the fusion of GAS into a distributed training algorithm \citep{Ma/etal/2019,Zhu/etal/2016,Tripathy/etal/2020,Wan/etal/2020,Angerd/etal/2020,Zheng/etal/2020}, and to extend our framework in accessing histories from disk storage rather than CPU memory.
Overall, we hope that our findings lead to the development of sophisticated and expressive GNNs evaluated on large-scale graphs.

\section*{Acknowledgements}

This work has been supported by the \emph{German Research Association (DFG)} within the Collaborative Research Center SFB 876 \emph{Providing Information by Resource-Constrained Analysis}, projects A6 and B2.

\bibliography{paper}
\bibliographystyle{icml2021}

\newpage
\onecolumn

\icmltitle{Appendix}

\setcounter{theorem}{0}

\section{Proofs}%
\label{sec:proofs}

\begin{lemma}\label{lemma1}
  Let $\textsc{Message}^{(\ell)}_{\mat{\theta}}$ and $\textsc{Update}^{(\ell)}_{\mat{\theta}}$ be Lipschitz continuous functions with Lipschitz constants $k_1$ and $k_2$, respectively.
  If, for all $v \in \mathcal{V}$, the inputs are close to the exact input, \ie~$\| \vec{\tilde{h}}^{(\ell - 1)}_{v} - \vec{h}^{(\ell - 1)}_{v} \| \leq \delta$, and the historical embeddings do not run too stale, \ie~$\| \vec{\bar{h}}^{(\ell - 1)}_v - \vec{\tilde{h}}^{(\ell - 1)}_v \| \leq \epsilon$, then the output error is bounded by
  \begin{equation*}
    \| \vec{\tilde{h}}^{(\ell)}_v - \vec{h}^{(\ell)}_v \| \leq \delta\,k_2 + (\delta + \epsilon)\,k_1\,k_2\,|\mathcal{N}(v)|.
  \end{equation*}
\end{lemma}

\begin{proof}
  By triangular inequality, it holds that $\| \vec{\bar{h}}^{(\ell - 1)}_v - \vec{h}^{(\ell - 1)}_v \| \leq \delta + \epsilon$.
  Since both $\textsc{Message}^{(\ell)}_{\mat{\theta}}$ and $\textsc{Update}^{(\ell)}_{\mat{\theta}}$ denote Lipschitz continuous functions with Lipschitz constants $k_1$ and $k_2$, respectively, it further holds that for any $\vec{x}, \vec{y}$:
  \begin{align*}
    \| \textsc{Message}^{(\ell)}_{\mat{\theta}}(\vec{x}) - \textsc{Message}^{(\ell)}_{\mat{\theta}}(\vec{y}) \| & \leq k_1\| \vec{x} - \vec{y} \| \\
    \| \textsc{Update}^{(\ell)}_{\mat{\theta}}(\vec{x}) - \textsc{Update}^{(\ell)}_{\mat{\theta}}(\vec{y}) \| & \leq k_2 \| \vec{x} - \vec{y} \|
  \end{align*}
  Furthermore, the Lipschitz constants for the aggregations $\sum_{\vec{x} \in \mathcal{X}} \vec{x}$, $\frac{1}{|\mathcal{X}|} \sum_{\vec{x} \in \mathcal{X}} \vec{x}$ and $\max_{\vec{x} \in \mathcal{X}} \vec{x}$ are given as $|\mathcal{X}|$, $1$ and $1$, respectively.
  Then,
  \begin{align*}
    & \| \textsc{Update}^{(\ell)}_{\mat{\theta}}(\vec{\tilde{h}}^{(\ell - 1)}_v, \bigoplus\limits_{\mathclap{w \in \mathcal{N}(v)}} \textsc{Message}^{(\ell)}_{\mat{\theta}}(\vec{\bar{h}}_w^{(\ell - 1)})) - \textsc{Update}^{(\ell)}_{\mat{\theta}}(\vec{h}^{(\ell - 1)}_v, \bigoplus\limits_{\mathclap{w \in \mathcal{N}(v)}} \textsc{Message}^{(\ell)}_{\mat{\theta}}(\vec{h}_w^{(\ell - 1)})) \| \\
    \leq\,& k_2 \, ( \delta + |\mathcal{N}(v)| \, (k_1 \, (\delta + \epsilon))) = \delta\,k_2 + (\delta + \epsilon)\,k_1\,k_2\,|\mathcal{N}(v)|. \qedhere
  \end{align*}
\end{proof}

\begin{theorem}\label{theorem1}
  Let $\vec{f}^{(L)}_{\mat{\theta}}$ be a $L$-layered GNN, containing only Lipschitz continuous $\textsc{Message}^{(\ell)}_{\mat{\theta}}$ and $\textsc{Update}^{(\ell)}_{\mat{\theta}}$ functions with Lipschitz constants $k_1$ and $k_2$, respectively.
  If, for all $v \in \mathcal{V}$ and all $\ell \in \{1, \ldots, L-1\}$, the historical embeddings do not run too stale, \ie~$\| \vec{\bar{h}}^{(\ell)}_v - \vec{\tilde{h}}^{(\ell)}_v \| \leq \epsilon^{(\ell)}$, then the final output error is bounded by
  \begin{equation*}
    \| \vec{\tilde{h}}_{v,j}^{(L)} - \vec{h}_{v,j}^{(L)} \| \leq \sum_{\ell = 1}^{L-1} \epsilon^{(\ell)} \, k_1^{L - \ell} \, k_2^{L - \ell} \, {|\mathcal{N}(v)|}^{L - \ell}.
  \end{equation*}
\end{theorem}

\begin{proof}
  For layer $\ell = 1$, the inputs do not need to be estimated, \ie~$\delta^{(0)} = \| \vec{\tilde{h}}_v^{(0)} - \vec{h}_v^{(0)} \| = 0$, and, as a result, the output is \emph{exact}, \ie~$\delta^{(1)} = \| \vec{\tilde{h}}_v^{(1)} - \vec{h}_v^{(1)} \| = 0$.
  With $\| \vec{\bar{h}}^{(1)}_v - \vec{\tilde{h}}^{(1)}_v \| \leq \epsilon^{(1)}$, it directly follows via Lemma~\ref{lemma1} that the approximation error of layer $\ell =2$ is bounded by $\| \vec{\tilde{h}}^{(2)}_v - \vec{h}^{(2)}_v \| \leq \epsilon^{(1)} \, k_1 \, k_2 \, |\mathcal{N}(v)| = \delta^{(2)}$.
  Recursively replacing
  \begin{equation*}
    \delta^{(\ell)} = \delta^{(\ell -1)}\,k_2 + (\delta^{(\ell-1)} + \epsilon^{(\ell-1)}) \, k_1 \, k_2 \, |\mathcal{N}(v)|
  \end{equation*}
  in $\| \vec{\tilde{h}}_v^{(L)} - \vec{h}_v^{(L)} \| \leq \delta^{(L-1)}\,k_2 + (\delta^{(L-1)} + \epsilon^{(L-1)}) \, k_1 \, k_2 \, |\mathcal{N}(v)|$ (\cf~Lemma~\ref{lemma1}) yields
  \begin{equation*}
    \| \vec{\tilde{h}}_v^{(L)} - \vec{h}_v^{(L)} \| \leq \sum_{\ell = 1}^{L-1} \epsilon^{(\ell)} \, k_1^{L - \ell} \, k_2^{L - \ell} \, {|\mathcal{N}(v)|}^{L - \ell}. \qedhere
  \end{equation*}
\end{proof}

\begin{minipage}[H]{0.525\textwidth}
\begin{proposition}
  Let $\vec{f}^{(L)}_{\mat{\theta}} \colon \mathcal{V} \to \mathbb{R}^{d}$ be a $L$-layered GNN as expressive as the WL test in distinguishing the $L$-hop neighborhood around each node $v \in \mathcal{V}$.
  Then, there exists a graph $\mat{A} \in {\{0, 1 \}}^{|\mathcal{V}| \times |\mathcal{V}|}$ for which $\vec{f}^{(L)}_{\mat{\theta}}$ operating on a sampled variant $\mat{\tilde{A}}$, $\tilde{a}_{v,w} = \begin{cases} \frac{|\mathcal{N}(v)|}{|\mathcal{\tilde{N}}(v)|}, & \textrm{if } w \in \mathcal{\tilde{N}}(v) \\ 0, & \textrm{otherwise} \end{cases}$, produces a non-equivalent coloring, \ie~$\vec{\tilde{h}}^{(L)}_v \neq \vec{\tilde{h}}^{(L)}_w$ while $c_v^{(L)} = c_w^{(L)}$ for nodes $v, w \in \mathcal{V}$.
\end{proposition}

\begin{proof}
  Consider the colored graph $\mat{A}$ and its sampled variant $\mat{\tilde{A}}$ as shown on the right.
  Here, it holds that $\vec{h}^{(1)}_{v_1} = \vec{h}^{(1)}_{v_4}$ while $\vec{\tilde{h}}^{(1)}_{v_1} \neq \vec{\tilde{h}}^{(1)}_{v_4}$.
\end{proof}
\end{minipage}
\hfill{}
\begin{minipage}[H]{0.45\textwidth}
  {\includegraphics[width=\textwidth]{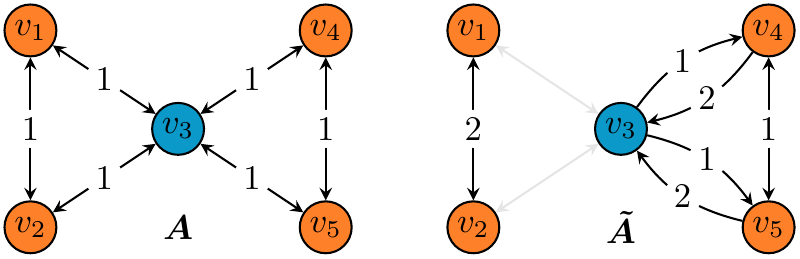}}
\end{minipage}

\begin{minipage}[H]{0.625\textwidth}
\begin{lemma}\label{lemma2}
  Let $\{ \hspace{-3pt} \{ \vec{h}_v^{(\ell - 1)} \colon v \in \mathcal{V} \} \hspace{-3pt} \}$ be a countable multiset such that $\| \vec{h}_v^{(\ell - 1)} - \vec{h}_w^{(\ell - 1)} \| > 2 (\delta + \epsilon)$ for all $v,w \in \mathcal{V}$, $\vec{h}_v^{(\ell - 1)} \neq \vec{h}_w^{(\ell - 1)}$.
  If the inputs are close to the exact input, \ie~\mbox{$\| \vec{\tilde{h}}_v^{(\ell - 1)} - \vec{h}_v^{(\ell - 1)} \| \leq \delta$}, and the historical embeddings do not run too stale, \ie~\mbox{$\| \vec{\bar{h}}_v^{(\ell - 1)} - \vec{\tilde{h}}_v^{(\ell - 1)} \| \leq \epsilon$}, then there exist $\textsc{Message}^{(\ell)}_{\mat{\theta}}$ and $\textsc{Update}^{(\ell)}_{\mat{\theta}}$ functions, such that
  \begin{equation*}
    \| \vec{f}^{(\ell)}_{\mat{\theta}}(\vec{\tilde{h}}_v^{(\ell - 1)}) - \vec{f}^{(\ell)}_{\mat{\theta}}(\vec{h}_v^{(\ell - 1)}) \| \leq \delta + \epsilon
  \end{equation*}
  and
  \begin{equation*}
    \| \vec{f}^{(\ell)}_{\mat{\theta}}(\vec{h}_v^{(\ell - 1)}) - \vec{f}^{(\ell)}_{\mat{\theta}}(\vec{h}_w^{(\ell - 1)}) \| > 2(\delta + \epsilon + \lambda)
  \end{equation*}
  for all $v,w \in \mathcal{V}$, $\vec{h}_v^{(\ell - 1)} \neq \vec{h}_w^{(\ell - 1)}$ and all $\lambda > 0$.
\end{lemma}

\begin{proof}
  Define $\vec{\phi} \colon \mathbb{R}^d \to \mathbb{R}^d$ as the Voronoi tessellation induced by exact inputs $\{ \vec{h}_v^{(\ell - 1)} \colon v \in \mathcal{V} \}$:
  \begin{equation}
    \vec{\phi}(\vec{x}) = \vec{h}_v^{(\ell - 1)} \quad \textrm{if} \quad \| \vec{x} - \vec{h}_v^{(\ell - 1)} \| \leq \| \vec{x} - \vec{h}_w^{(\ell - 1)} \| \quad \textrm{for all } v \neq w \in \mathcal{V}
  \end{equation}
  Furthermore, we know that there exists $\textrm{Message}^{(\ell)}_{\mat{\theta}}$ and $\textsc{Update}^{(\ell)}_{\mat{\theta}}$ functions so that $\vec{f}^{(\ell)}_{\mat{\theta}}$ is injective for all countable multisets \citep{Zaheer/etal/2017,Xu/etal/2019,Morris/etal/2019,Maron/etal/2019}.
  Therefore, it holds that $\| \vec{f}^{(\ell)}_{\mat{\theta}}\big(\vec{\phi}\big(\vec{\tilde{h}}_v^{(\ell - 1)}\big)\big) - \vec{f}^{(\ell)}_{\mat{\theta}}\big(\vec{\phi}\big(\vec{h}_v^{(\ell - 1)}\big)\big) \| = 0 \leq \delta$.
  Since $\{ \hspace{-3pt} \{ \vec{h}_v^{(\ell - 1)} \colon v \in \mathcal{V} \} \hspace{-3pt} \}$ is countable and $\vec{f}^{(\ell)}_{\mat{\theta}}$ is injective, there exists a $\kappa > 0$ such that
  $\| \vec{f}^{(\ell)}_{\mat{\theta}}\big(\vec{\phi}\big(\vec{h}_v^{(\ell - 1)}\big)\big) - \vec{f}^{(\ell)}_{\mat{\theta}}\big(\vec{\phi}\big(\vec{h}_w^{(\ell - 1)}\big)\big) \| > \kappa$ for all $v,w \in \mathcal{V}$, $\vec{h}_v^{(\ell - 1)} \neq \vec{h}_w^{(\ell - 1)}$.
  Due to the homogeneity of $\| \cdot \|$, it directly follows that there must exists $\alpha > 0$ so that
  \begin{equation*}
    \| \alpha \vec{f}^{(\ell)}_{\mat{\theta}}\big(\vec{\phi}\big(\vec{h}_v^{(\ell - 1)}\big)\big) - \alpha \vec{f}^{(\ell)}_{\mat{\theta}}\big(\vec{\phi}\big(\vec{h}_w^{(\ell - 1)}\big)\big) \| > \alpha\,\kappa \geq 2(\delta  + \epsilon + \lambda)
  \end{equation*}
 for all $v,w \in \mathcal{V}$, $\vec{h}_v^{(\ell - 1)} \neq \vec{h}_w^{(\ell - 1)}$ and all $\lambda > 0$.
\end{proof}
\end{minipage}
\hfill{}
\begin{minipage}[H]{0.35\textwidth}
  \includegraphics[width=\textwidth]{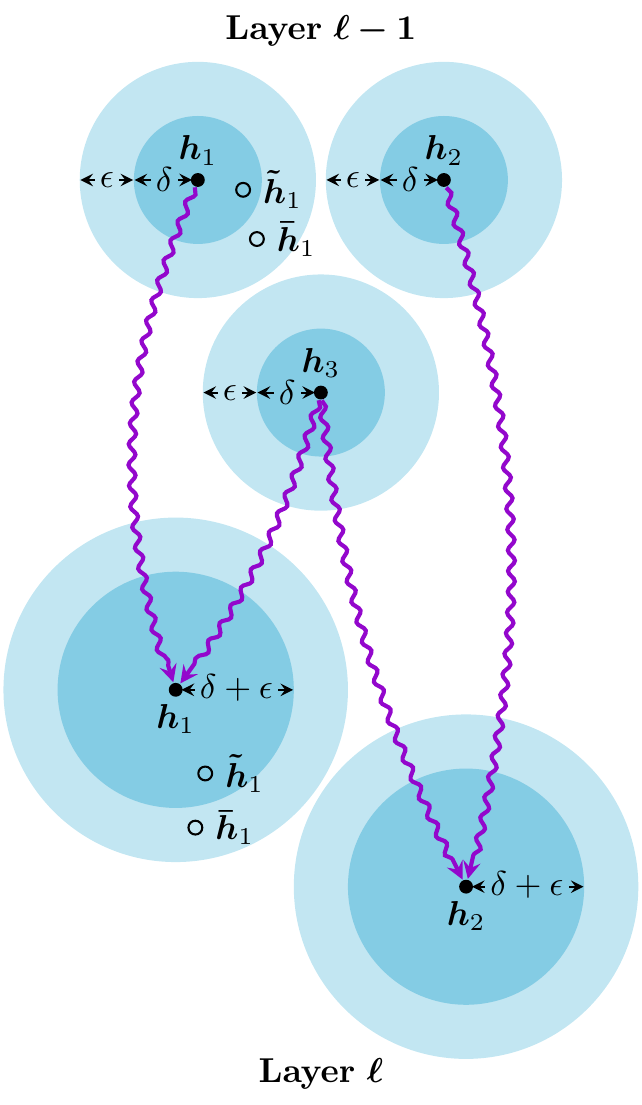}
\end{minipage}

\begin{theorem}\label{theorem2}
  Let $\vec{f}^{(L)}_{\mat{\theta}}$ be a $L$-layered GNN in which all $\textsc{Message}^{(\ell)}_{\mat{\theta}}$ and $\textsc{Update}^{(\ell)}_{\mat{\theta}}$ functions fulfill the conditions of Lemma~\ref{lemma2}.
  Then, there exists a map $\phi \colon \mathbb{R}^d \to \Sigma$ so that $\phi(\vec{\tilde{h}}^{(L)}_v) = c^{(L)}_v$ for all $v \in \mathcal{V}$.
\end{theorem}

\begin{proof}
  Define $\phi \colon \mathbb{R}^d \to \Sigma$ as the Voronoi tessellation induced by exact outputs $\{ \vec{h}_v^{(L)} \colon v \in \mathcal{V} \}$:
  \begin{equation*}
    \phi(\vec{x}) = c_v^{(L)} \quad \textrm{if} \quad \| \vec{x} - \vec{h}^{(L)}_v \| \leq \| \vec{x} - \vec{h}^{(L)}_w \| \quad \textrm{for all } v \neq w \in \mathcal{V}
  \end{equation*}
  Since each GNN layer $\vec{f}^{(\ell)}_{\mat{\theta}}$ is injective for exact inputs,
  we know that such a function needs to exist \citep{Xu/etal/2019,Morris/etal/2019}.
  Therefore, it is sufficient to show that there exists a $\delta^{(L)} > 0$ so that $\| \vec{\tilde{h}}^{(L)}_v - \vec{h}^{(L)}_v \| \leq \delta^{(L)}$ and $\| \vec{h}^{(L)}_v - \vec{h}^{(L)}_w \| > 2 \delta^{(L)}$ for all $v, w \in \mathcal{V}$, $\vec{h}^{(L)}_v \neq \vec{h}^{(L)}_w$.
  Following upon Theorem~\ref{theorem1}, we know that $\| \vec{\tilde{h}}_v^{(1)} - \vec{h}_v^{(1)} \| = 0$.
  Due to Lemma~\ref{lemma2}, it holds that $\| \vec{\tilde{h}}_v^{(2)} - \vec{h}_v^{(2)} \| \leq \epsilon^{(1)}$.
  The next layer introduces an increased error, \ie~$\| \vec{\bar{h}}_v^{(2)} - \vec{h}_v^{(2)} \| \leq \epsilon^{(1)} + \epsilon^{(2)}$, and to compensate, we set $\lambda^{(2)} = \epsilon^{(2)}$ so that $\| \vec{h}_v^{(2)} - \vec{h}_w^{(2)} \| > 2 \, (\epsilon^{(1)} + \epsilon^{(2)} )$ for all $v, w \in \mathcal{V}$, $\vec{h}^{(L)}_v \neq \vec{h}^{(L)}_w$.
  By recursively applying Lemma~\ref{lemma2} with $\lambda^{(\ell)} = \epsilon^{(\ell)}$, it immediately follows that
  $\| \vec{\tilde{h}}_v^{(L)} - \vec{h}_v^{(L)} \| \leq \sum_{\ell = 1}^{L-1} \epsilon^{(\ell)} = \delta^{(L)}$, and $\| \vec{\tilde{h}}_v^{(L)} - \vec{h}_w^{(L)} \| > \sum_{\ell = 1}^{L-1} 2\,\epsilon^{(\ell)}$  for all $v, w \in \mathcal{V}$, \mbox{$\vec{h}^{(L)}_v \neq \vec{h}^{(L)}_w$}.
\end{proof}

\section{Algorithm}%
\label{sec:algorithm}

Our GAS mini-batch training algorithm is given in Algorithm~\ref{alg:algorithm}:

\begin{algorithm}[H]
  \setstretch{1.5}
  \caption{GAS Mini-batch Execution}\label{alg:algorithm}
  \begin{algorithmic}
    \STATE{\textbf{Input}: Graph $\mathcal{G} = (\mathcal{V}, \mathcal{E})$, input node features $\mat{H}^{(0)}$, number of batches $B$, number of layers $L$}
    \STATE{$\{ \mathcal{B}_1, \ldots, \mathcal{B}_B \} \leftarrow \textsc{Split}(\mathcal{G}, B)$}
    \STATE{$\mathcal{V}_b \leftarrow \bigcup_{v \in \mathcal{B}_b} \mathcal{N}(v) \cup \{ v \}$} \hfill{} $\forall b \in \{ 1, \ldots, B \}$
    \STATE{$\mathcal{G}_b \leftarrow \mathcal{G}[\mathcal{V}_b]$} \hfill{} $\forall b \in \{ 1, \ldots, B \}$
    \FOR{$\mathcal{B}_b \in \{ \mathcal{B}_1, \ldots, \mathcal{B}_B \}$}
      \FOR{$\ell \in \{ 1, \ldots, L - 1 \}$}
        \STATE{$\vec{h}_v^{(\ell)} \leftarrow \vec{f}^{(\ell)}_{\mat{\theta}}( \vec{h}_v^{(\ell - 1)}, \{ \hspace{-3pt} \{ \vec{h}_w^{(\ell - 1)} : w \in \mathcal{N}(v) \} \hspace{-3pt} \} )$} \hfill{} $\forall v \in \mathcal{B}_b$
        \STATE{$\textsc{Push}^{(\ell)}(\vec{h}_v^{(\ell)})$} \hfill{} $\forall v \in \mathcal{B}_b$
        \STATE{$\vec{h}_w^{(\ell)} \leftarrow \textsc{Pull}^{(\ell)}(w)$} \hfill{} $\forall w \in \mathcal{V}_b \setminus \mathcal{B}_b$
      \ENDFOR%
      \STATE{$\vec{h}_v^{(L)} \leftarrow \vec{f}^{(L)}_{\mat{\theta}}( \vec{h}_v^{(L - 1)}, \{ \hspace{-3pt} \{ \vec{h}_w^{(L - 1)} : w \in \mathcal{N}(v) \} \hspace{-3pt} \} )$} \hfill{} $\forall v \in \mathcal{B}_b$
    \ENDFOR%
  \end{algorithmic}
\end{algorithm}

\section{GNN Operators}%
\label{sec:gnn_operators}

We briefly recap the details of all graph convolutional layers used in our experiments.
We omit final non-linearities and edge features due to simplicity.

\paragraph{Graph Convolutional Networks (\textsc{GCN})}%
\label{par:graph_convolutional_networks}

use a symmetrically normalized mean aggregation followed by linear transformation \citep{Kipf/Welling/2017}
\begin{equation*}
  \vec{h}_v^{(\ell)} = \sum_{w \in \mathcal{N}(v) \cup \{ v \}} \frac{1}{c_{w,v}} \mat{W} \vec{h}_w^{(\ell - 1)},
\end{equation*}
where $c_{w,v} = \sqrt{\deg(w) + 1} \sqrt{\deg(v) + 1}$.

\paragraph{Graph Attention Networks (\textsc{GAT})}%
\label{par:graph_attention_networks}

perform an anisotropic aggregation \citep{Velickovic/etal/2018}
\begin{equation*}
  \vec{h}_v^{(\ell)} = \sum_{w \in \mathcal{N}(v) \cup \{ v \}} \alpha_{w,v} \mat{W} \vec{h}_w^{(\ell - 1)},
\end{equation*}
where normalization is achieved via learnable attention coefficients
\begin{equation*}
  \alpha_{w,v} = \frac{\exp \left( \textrm{LeakyReLU}\left( \vec{a}^{\top} \hspace{-2pt} \left[ \mat{W} \vec{h}_v^{(\ell - 1)}, \mat{W} \vec{h}_w^{(\ell - 1)} \right] \right) \right)}{\sum_{k \in \mathcal{N}(v) \cup \{ v \}}\exp \left( \textrm{LeakyReLU}\left( \vec{a}^{\top} \hspace{-2pt} \left[ \mat{W} \vec{h}_v^{(\ell - 1)}, \mat{W} \vec{h}_k^{(\ell - 1)} \right] \right) \right)}.
\end{equation*}

\paragraph{Approximate Personalized Propagation of Neural Predictions (\textsc{APPNP})}%
\label{par:approximate_personalized_propagation_of_neural_predictions}

networks first perform a graph-agnostic prediction of node labels, \ie~$\vec{h}^{(0)}_v = \textrm{MLP}(\vec{x}_v)$, and smooth initial label predictions via propagation afterwards \citep{Klicpera/etal/2019a}
\begin{equation*}
  \vec{h}^{(\ell)} = \alpha\,\vec{h}^{(0)} + (1 - \alpha) \sum_{w \in \mathcal{N} \cup \{ v \}} \frac{1}{c_{w,v}} \vec{h}^{(\ell - 1)}_w,
\end{equation*}
where $\alpha \in [0, 1]$ denotes the teleport probability and $c_{w,v}$ is defined as in \textsc{GCN}.
Notably, the final propagation layers are non-trainable, and predictions are solely conditioned on node features (while gradients of model parameters are not).

\paragraph{Simple and Deep Graph Convolutional Networks (\textsc{GCNII})}%
\label{par:simple_and_deep_graph_convolutional_networks}

extend the idea of \textsc{APPNP} to a trainable propgation scheme which leverages initial residual connections \citep{Chen/etal/2020}
\begin{equation*}
  \vec{h}_v^{(\ell)} = \alpha \mat{W} \vec{h}_v^{(0)} + (1 - \alpha) \sum_{w \in \mathcal{N}(v) \cup \{ v \}} \frac{1}{c_{w,v}} \mat{W} \vec{h}_w^{(\ell - 1)},
\end{equation*}
and $\mat{W}$ makes use of identity maps, \ie~$\mat{W} \leftarrow (1 - \beta)\mat{I} + \beta \mat{W}$ for $\beta \in [0, 1]$.

\paragraph{Graph Isomorphism Networks (\textsc{GIN})}%
\label{par:graph_isomorphism_networks}

make use of sum aggregation and MLPs to obtain a maximally powerful GNN operator \citep{Xu/etal/2019}
\begin{equation*}
  \vec{h}_v^{(\ell)} = \textrm{MLP}_{\mat{\theta}} \left( (1 + \epsilon) \, \vec{h}_v^{(\ell - 1)} + \sum_{w \in \mathcal{N}(v)} \vec{h}_w^{(\ell - 1)} \right),
\end{equation*}
where $\epsilon \in \mathbb{R}$ is a trainable parameter in order to distinguish neighbors from central nodes.

\paragraph{Principal Neighborhood Aggregation (\textsc{PNA})}%
\label{par:principal_neighborhood_aggregation}

networks leverage mulitple aggregators combined with degree-scalers to capture graph structural properties \citep{Corso/etal/2020}
\begin{equation*}
  \vec{h}_v^{(\ell)} = \mat{W}_2 \left[ \vec{h}_v^{(\ell - 1)}, \bigoplus_{w \in \mathcal{N}(v)} \mat{W}_1 \left[ \vec{h}_v^{(\ell - 1)}, \vec{h}_w^{(\ell - 1)} \right] \right],
\end{equation*}
where
\begin{equation*}
  \bigoplus =
  \underbrace{\begin{bmatrix}
    1 \\
    s(\deg(v), 1) \\
    s(\deg(v), -1)
  \end{bmatrix} }_{\text{Scalers}}
  \otimes
  \underbrace{\begin{bmatrix}
    \textrm{mean} \\
    \min \\
    \max
  \end{bmatrix}}_{\text{Aggregators}},
\end{equation*}
with $\otimes$ being the tensor product and
\begin{equation*}
  s(d, \alpha) = {\left( \frac{\log(d + 1)}{\frac{1}{|\mathcal{V}|} \sum_{v \in \mathcal{V}} \log(\deg(v) + 1)} \right)}^{\alpha}
\end{equation*}
denoting degree-scalers.

\section{PyGAS Programming Interface}%
\label{sec:programming_interface}

To highlight the ease-of-use of our framework, we showcase the necessary changes to convert a common \textsc{GCN} architecture \citep{Kipf/Welling/2017} implemented in \textsc{PyTorch Geometric} \citep{Fey/Lenssen/2019} (\cf~Listing~\ref{lst:code1}) to its corresponding scalable version (\cf~Listing~\ref{lst:code2}).
In particular, our model now inherits from \texttt{ScalableGNN}, which takes care of creating all history embeddings (accessible via \texttt{self.histories}) and provides an efficient concurrent history access pattern via \texttt{push\_and\_pull()}.
Notably, the \texttt{forward()} execution method of our model now takes in the additional \texttt{n\_id} parameter, which holds the global node index for each node in the current mini-batch.
This assignment vector is necessary to push and pull the intermediate mini-batch embeddings to and from the global history embeddings.

\begin{lstlisting}[caption={\textbf{Full-batch} \textsc{GCN} \citep{Kipf/Welling/2017} \textbf{model within} \textsc{PyTorch Geometric} \citep{Fey/Lenssen/2019}\textbf{.}},label=lst:code1]
from torch_geometric.nn import GCNConv

class GNN(Module):
    def __init__(self, in_channels, hidden_channels, out_channels, num_layers):
        super(GNN, self).__init__()

        self.convs = ModuleList()
        self.convs.append(GCNConv(in_channels, hidden_channels))
        for _ in range(num_layers - 2):
            self.convs.append(GCNConv(hidden_channels, hidden_channels))
        self.convs.append(GCNConv(hidden_channels, out_channels))

    def forward(self, x, adj_t):
        for conv in self.convs[:-1]:
            x = conv(x, adj_t).relu()
        return self.convs[-1](x, adj_t)
\end{lstlisting}

\begin{lstlisting}
from torch_geometric.nn import GCNConv
\end{lstlisting}
\vspace{-\baselineskip}
\begin{lstlisting}[backgroundcolor=\color{codebg2}]
from torch_geometric_autoscale import ScalableGNN
\end{lstlisting}
\vspace{-\baselineskip}
\begin{lstlisting}
w
\end{lstlisting}
\vspace{-\baselineskip}
\begin{lstlisting}[backgroundcolor=\color{codebg1}]
class GNN(ScalableGNN):
    def __init__(self, num_nodes, in_channels, hidden_channels, out_channels, num_layers):
        super(GNN, self).__init__(num_nodes, hidden_channels, num_layers)
\end{lstlisting}
\vspace{-\baselineskip}
\begin{lstlisting}

        self.convs = ModuleList()
        self.convs.append(GCNConv(in_channels, hidden_channels))
        for _ in range(num_layers - 2):
            self.convs.append(GCNConv(hidden_channels, hidden_channels))
        self.convs.append(GCNConv(hidden_channels, out_channels))
\end{lstlisting}
\vspace{-\baselineskip}
\begin{lstlisting}
w
\end{lstlisting}
\vspace{-\baselineskip}
\begin{lstlisting}[backgroundcolor=\color{codebg1}]
    def forward(self, x, adj_t, n_id):
        for conv, history in zip(self.convs[:-1], self.histories):
\end{lstlisting}
\vspace{-\baselineskip}
\begin{lstlisting}
            x = conv(x, adj_t).relu()
\end{lstlisting}
\vspace{-\baselineskip}
\begin{lstlisting}[backgroundcolor=\color{codebg2}]
            x = self.push_and_pull(history, x, n_id)
\end{lstlisting}
\vspace{-\baselineskip}
\begin{lstlisting}[caption={\textbf{Mini-batch} \textsc{GCN} \citep{Kipf/Welling/2017} \textbf{model within} \textsc{PyTorch Geometric} \citep{Fey/Lenssen/2019} \textbf{and our proposed \emph{PyGAS} framework.} \textcolor{codebg1}{$\blacksquare$} denotes lines that require changes, while \textcolor{codebg2}{$\blacksquare$} refers to newly added lines. Only minimal changes are required to auto-scale \textsc{GCN} (or any other model) to large graphs.},label=lst:code2]
        return self.convs[-1](x, adj_t)
\end{lstlisting}

\section{Addtional Ablation Studies}%
\label{sec:addtional_ablation_studies}

We report additional ablation studies to further strengthen the motivation of our GAS framework:

\begin{table}
  \centering
  \caption{%
    \textbf{Inter-/intra-connectivity ratio for real-world datasets with different mini-batch sampling strategies.}
    Utilizing \textsc{Metis} heavily minimizes inter-connectivity between mini-batches, which reduces history accesses and tightens approximation errors in return.
  }\label{tab:metis}
  \setlength{\tabcolsep}{5pt}
  \resizebox{0.8\linewidth}{!}{%
  \begin{tabular}{lrrrrrrrrrrrrrrr}
    \toprule
    \textbf{Sampling} & \mr{2}{\textsc{Cora}} & \mr{2}{\textsc{CiteSeer}} & \mr{2}{\textsc{PubMed}} & \mc{2}{r}{\textsc{Coauthor-~~~}} & \mc{2}{c}{\textsc{~~~~~~Amazon-}} & \mr{2}{\textsc{Wiki-CS}} \\
    \textbf{Scheme} & & & & \textsc{CS} & \textsc{Physics} & \textsc{Computer} & \textsc{Photo} \\
    \midrule
    Random         & 1.33 & 1.24 & 3.17 & 6.81 & 9.94 & 9.05 & 5.61 & 5.85 \\
    \textsc{Metis} & 0.14 & 0.02 & 0.52 & 2.77 & 2.26 & 2.27 & 1.03 & 1.12 \\
    \cmidrule{2-9}
    & \mr{2}{\textsc{Cluster}} & \mr{2}{\textsc{Pattern}} & \mr{2}{\textsc{Reddit}} & \mr{2}{\textsc{PPI}} & \mr{2}{\textsc{Flickr}} & \mr{2}{\textsc{Yelp}} & \texttt{ogbn-} & \texttt{ogbn-~} \\
    & & & & & & & \texttt{arxiv} & \texttt{products} \\
    \cmidrule{2-9}
    Random         & 36.64 & 51.02 & 6.58 & 6.79 & 1.82 & 6.74 & 3.02 & 26.18 \\
    \textsc{Metis} &  1.57 &  1.61 & 2.80 & 1.27 & 1.07 & 2.52 & 0.48 &  1.94 \\
    \bottomrule
  \end{tabular}}
\end{table}

\paragraph{Minimizing Inter-Connectivity Between Batches.}%

We make use of graph clustering methods \citep{Karypis/Kumar/1998,Dhillon/etal/2007} in order to minimize the inter-connectivity between batches, which minimizes history accesses and therefore increases closeness and reduces staleness in return.
To evaluate this impact in practice, Tabel~\ref{tab:metis} lists the inter-/intra-connectivity ratio of all real-world datasets used in our experiments, both for randomly sampled mini-batches as well as for utilizing \textsc{Metis} partitions as mini-batches.
Notably, applying \textsc{Metis} beforehand reduces the overall inter-/intra-connectivity ratio by a factor of 4 on average, which results in only a fraction of history accesses.
Furthermore, most real-world datasets come with inter-/intra-connectivity ratios between $0.1$ and $2.5$, leading to only marginal runtime overheads when leveraging historical information, as confirmed by our runtime analysis.

\begin{table}
  \centering
  \caption{%
    \textbf{Ablation study for a 4-layer} \textsc{GIN} \citep{Xu/etal/2019} \textbf{model on the} \textsc{Cluster} \textbf{dataset} \citep{Dwivedi/etal/2020}\textbf{.}
    Combining both GAS techniques help in resembling full-batch performance for expressive models with highly non-linear message passing phases.
  }\label{tab:ablation_study}
  \resizebox{0.65\linewidth}{!}{%
  \begin{tabular}{cccccc}
    \toprule
    & & & \mc{3}{c}{\textbf{Accuracy}} \\
    & & & Training & Validation & Test \\
    \midrule
    & \mc{2}{c}{\textbf{Full-batch Baseline}}                               & 60.49 & 58.17 & 58.49 \\
    \midrule
    & \textbf{Minimizing} & \textbf{Enforcing} \\
    & \textbf{Inter-Connectivity} & \textbf{Lipschitz Continuity} \\
    \midrule
    \mr{3}{\rotatebox{90}{\small{\textbf{GAS}}}}
    & \textcolor{codered}{\XSolidBrush} & \textcolor{codered}{\XSolidBrush} & 55.66 & 54.86 & 55.15 \\
    & \textcolor{green}{\CheckmarkBold} & \textcolor{codered}{\XSolidBrush} & 58.97 & 57.79 & 57.82 \\
    & \textcolor{green}{\CheckmarkBold} & \textcolor{green}{\CheckmarkBold} & \textbf{60.67} & \textbf{58.21} & \textbf{58.51} \\
    \bottomrule
  \end{tabular}}
\end{table}

\paragraph{Analysis of Gains for Obtaining Expressive Node Representations.}%

Next, we highlight the impacts of minimizing the inter-connectivity between mini-batches and enforcing Lipschitz continuity of the learned function in order to derive expressive node representations.
Here, we benchmark a 4-layer GIN model \citep{Xu/etal/2019} on the \textsc{Cluster} dataset \citep{Dwivedi/etal/2020}, \cf~Table~\ref{tab:ablation_study}.
Notably, both solutions achieve significant gains in training, validation and test performance, and together, they are able to closely resemble the performance of full-batch training.
However, we found that Lipschitz continuity regularization only helps in non-linear message passing phases, while it does not provide any additional gains for linear operators such as \textsc{GCN} \citep{Kipf/Welling/2017}.

\section{Datasets}%
\label{sec:datasets}

We give detailed statistics for all datasets used in our experiments, \cf~Table~\ref{tab:dataset_statistics}, which include the following tasks:

\begin{table}
  \centering
  \caption{%
    \textbf{Dataset statistics.}
  }\label{tab:dataset_statistics}
  \resizebox{0.85\linewidth}{!}{%
  \begin{tabular}{llcrrrrr}
    \toprule
    & \textbf{Dataset} & \textbf{Task} & \textbf{Nodes} & \textbf{Edges} & \textbf{Features} & \textbf{Classes} & \textbf{Label Rate} \\
    \midrule
    \mr{8}{\rotatebox{90}{\small{Small-scale}}}
    & \textsc{Cora} & multi-class & 2,708 & 5,278 & 1,433 & 7 & 5.17\% \\
    & \textsc{CiteSeer} & multi-class & 3,327 & 4,552 & 3,703 & 6 & 3.61\% \\
    & \textsc{PubMed} & multi-class & 19,717 & 44,324 & 500 & 3 & 0.30\% \\
    & \textsc{Coauthor-CS} & multi-class & 18,333 & 81,894 & 6,805 & 15 & 1.64\% \\
    & \textsc{Coauthor-Physics} & multi-class & 34,493 & 247,962 & 8,415 & 5 & 0.29\% \\
    & \textsc{Amazon-Computer} & multi-class & 13,752 & 245,861 & 767 & 10 & 1.45\% \\
    & \textsc{Amazon-Photo} & multi-class & 7,650 & 119,081 & 745 & 8 & 2.09\% \\
    & \textsc{Wiki-CS} & multi-class & 11,701 & 215,863 & 300 & 10 & 4.96\% \\
    \midrule
    \mr{7}{\rotatebox{90}{\small{Large-scale}}}
    & \textsc{Cluster} & multi-class & 1,406,436 & 25,810,340 & 6 & 6 & 83.35\% \\
    & \textsc{Reddit} & multi-class & 232,965 & 11,606,919 & 602 & 41 & 65.86\% \\
    & \textsc{PPI} & multi-label & 56,944 & 793,632 & 50 & 121 & 78.86\% \\
    & \textsc{Flickr} & multi-class & 89,250 & 449,878 & 500 & 7 & 50.00\%  \\
    & \textsc{Yelp} & multi-label & 716,847 & 6,977,409 & 300 & 100 & 75.00\% \\
    & \texttt{ogbn-arxiv} & multi-class & 169,343 & 1,157,799 & 128 & 40 & 53.70\% \\
    & \texttt{ogbn-products} & multi-class & 2,449,029 & 61,859,076 & 100 & 47 & 8.03\% \\
    \bottomrule
  \end{tabular}}
\end{table}

\begin{enumerate}
  \setlength\itemsep{0em}
  \item classifying academic papers in citation networks (\textsc{Cora}, \textsc{CiteSeer}, \textsc{PubMed}) \citep{Sen/etal/2008,Yang/etal/2016}
  \item categorizing computer science articles in Wikipedia graphs (\textsc{Wiki-CS}) \citep{Mernyei/Cangea/2020}
  \item predicting active research fields of authors in co-authorshop graphs (\textsc{Coauthor-CS}, \textsc{Coauthor-Physics}) \citep{Shchur/etal/2018}
  \item predicting product categories in co-purchase graphs (\textsc{Amazon-Computer}, \textsc{Amazon-Photo}) \citep{Shchur/etal/2018}
  \item identifying community clusters in Stochastic Block Models (\textsc{Cluster}, \textsc{Pattern}) \citep{Dwivedi/etal/2020}
  \item predicting communities of online posts based on user comments (\textsc{Reddit}) \citep{Hamilton/etal/2017}
  \item classifying protein functions based on the interactions of human tissue proteins (\textsc{PPI}) \citep{Hamilton/etal/2017}
  \item categorizing types of images based on their descriptions and properties (\textsc{Flickr}) \citep{Zeng/etal/2020a}
  \item classifying business types based on customers and friendship relations (\textsc{Yelp}) \citep{Zeng/etal/2020a}
  \item predicting subject areas of \textsc{arXiv} Computer Science papers (\texttt{ogbn-arxiv}) \citep{Hu/etal/2020}
  \item predicting product categories in an \textsc{Amazon} product co-purchasing network (\texttt{ogbn-products}) \citep{Hu/etal/2020}
\end{enumerate}

\end{document}